\documentclass{article}

\usepackage{PRIMEarxiv}

\usepackage[utf8]{inputenc} 
\usepackage[T1]{fontenc}    
\usepackage{hyperref}       
\usepackage{url}            
\usepackage{booktabs}       
\usepackage{amsfonts}       
\usepackage{nicefrac}       
\usepackage{microtype}      
\usepackage{lipsum}
\usepackage{fancyhdr}       
\usepackage{graphicx}       
\graphicspath{{media/}}     

\usepackage{cite}
\usepackage{amsmath,amssymb,amsfonts}
\usepackage{amsthm}
\usepackage{algorithmic}
\usepackage{graphicx}
\usepackage{textcomp}
\usepackage{xcolor}
\usepackage{cuted}
\usepackage{mathtools}
\def\BibTeX{{\rm B\kern-.05em{\sc i\kern-.025em b}\kern-.08em
    T\kern-.1667em\lower.7ex\hbox{E}\kern-.125emX}}

\usepackage{array,multirow,multicol}
\usepackage{booktabs}
\usepackage{subfigure}
\usepackage{stfloats}
\usepackage{url}
\usepackage{verbatim}
\usepackage{longtable}
\usepackage{pifont}
\usepackage{wrapfig}
\usepackage{mathrsfs}
\usepackage{caption}
\usepackage{subcaption}
\usepackage{hyperref}

\newtheorem{assumption}{Assumption}[section]

\RequirePackage{algorithm}

\makeatletter
\newenvironment{breakablealgorithm}
  {
    \begin{center}
      \refstepcounter{algorithm}
      \hrule height.8pt depth0pt \kern2pt
      \parskip 0pt
      \renewcommand{\caption}[2][\relax]{
        {\raggedright\textbf{\fname@algorithm~\thealgorithm} ##2\par}%
        \ifx\relax##1\relax 
          \addcontentsline{loa}{algorithm}{\protect\numberline{\thealgorithm}##2}%
        \else 
          \addcontentsline{loa}{algorithm}{\protect\numberline{\thealgorithm}##1}%
        \fi
        \kern2pt\hrule\kern2pt
     }
  }
  {
     \kern2pt\hrule\relax
   \end{center}
  }
\makeatother

\newcommand{\cmark}{\ding{51}}%
\newcommand{\xmark}{\ding{55}}%
\newcommand{\nop}[1]{}

\newcommand{\Nan}[1]{{\color{purple}\textbf{Nan: }#1}}

\newcommand{\memo}[1]{{\color{purple}(#1)}}
\newcommand{\Wendy}[1]{{\color{brown}\textbf{Wendy: }#1}}

\usepackage{datetime}

\newcommand{\G}{\mathcal{G}}
\newcommand{\V}{\mathcal{V}}
\newcommand{\E}{\mathcal{E}}
\newcommand{\F}{\mathcal{F}}
\newcommand{\N}{\mathcal{N}}

\newcommand{\D}{\mathbf{D}}
\newcommand{\A}{\mathbf{A}}
\newcommand{\mH}{\mathbf{H}}
\newcommand{\W}{\mathbf{W}}
\newcommand{\w}{\mathbf{w}}
\newcommand{\Z}{\mathbf{Z}}

\newcommand{\X}{\mathbf{X}}
\newcommand{\Y}{\mathbf{Y}}
\newcommand{\mS}{s}

\newcommand{\bbP}{\mathbb{P}}
\newcommand{\bbE}{\mathbb{E}}

\newcommand{\inter}{H_{\text{inter}}}
\newcommand{\intra}{H_{\text{intra}}}
\newcommand{\EB}{B}
\newcommand{\Nzero}{N_0}
\newcommand{\None}{N_1}
\newcommand{\Ntotal}{N}
\newcommand{\SDZ}{\sigma_\Z}
\newcommand{\gammaFdot}{\gamma'_\F}
\newcommand{\theoitem}{\frac{\sqrt{\Nzero\,\None}}{{\SDZ}_{:, m}\Ntotal^2}}
\newcommand{\DAD}{\hat{\mathbf{D}}^{-\frac{1}{2}}\hat{\mathbf{A}} \hat{\mathbf{D}}^{-\frac{1}{2}}}
\newcommand{\jsold}{\mathrm{js}^i_t}

\newcommand{\loss}{\boldsymbol{\ell}}
\newcommand{\pbc}{\rho_{X,s}}
\newcommand{\pbcz}{\rho_{\Z,s}}

\newcommand{\ourmethod}{$\text{F}^2$GNN}
\newcommand{\server}{\mathcal{S}}

\newtheorem{theorem}{Theorem}
\newtheorem{lemma}[theorem]{Lemma}
\newtheorem{definition}[theorem]{Definition}

\usepackage[hang]{footmisc}
\title{Equipping Federated Graph Neural Networks with
Structure-aware Group Fairness
\nop{\thanks{\textit{\underline{Citation}}: 
\textbf{Authors. Title. Pages.... DOI:000000/11111.}}} 
}

\author{
  Nan Cui, Xiuling Wang, Wendy Hui Wang, Violet Chen, Yue Ning \\
  Stevens Institute of Technology \\
  Hoboken, NJ\\
  \texttt{\{ncui, xwang193, hwang4, vchen3, yue.ning\}@stevens.edu} \\
}

\begin{document}

\newtheoremstyle{mytheoremstyle}
    {\topsep} 
    {\topsep} 
    {} 
    {} 
    {\bfseries} 
    {} 
    {.5em} 
    {} 

\newtheorem*{mylemma}{Lemma} 
\newtheorem*{mytheorem}{Theorem}

\maketitle
\thispagestyle{plain}
\pagestyle{plain}

\begin{abstract}
Graph Neural Networks (GNNs) have been widely used for various types of graph data processing and analytical tasks in different domains. Training GNNs over centralized graph data can be infeasible due to privacy concerns and regulatory restrictions. Thus, federated learning (FL) becomes a trending solution to address this challenge in a distributed learning paradigm.  
However, as GNNs may inherit historical bias from training data and lead to discriminatory predictions, the bias of local models can be easily propagated to the global model in distributed settings. This poses a new challenge in mitigating bias in federated GNNs. To address this challenge, we propose \ourmethod, a \underline{F}air \underline{F}ederated \underline{G}raph \underline{N}eural
\underline{N}etwork, that enhances group fairness of federated GNNs. As bias can be sourced from both data and learning algorithms, \ourmethod\ aims to mitigate both types of bias under federated settings. 
First, we provide theoretical insights on the connection between data bias in a training graph and  statistical fairness metrics of the trained GNN models. Based on the theoretical analysis, we design \ourmethod\ which contains two key components: a {\em fairness-aware local model update scheme} that enhances group fairness of the local models on the client side, and a {\em fairness-weighted global model update scheme} that takes both data bias and fairness metrics of local models into consideration in the aggregation process. 
We evaluate \ourmethod\ empirically versus a number of baseline methods, and demonstrate that \ourmethod\ outperforms these baselines in terms of both fairness and model accuracy. 
The code is publicly accessible via the following link: \href{https://github.com/yuening-lab/F2GNN}{https://github.com/yuening-lab/F2GNN}.
\end{abstract}

\section{Introduction}

Graph neural networks (GNNs) have emerged as a formidable tool for generating meaningful node representations~\cite{grover2016node2vec, HamiltonYL17, KipfW17} and making accurate predictions on nodes by leveraging graph topologies~\cite{YuanYWLJ21, NouranizadehMR21}. GNNs have been studied in many applications such as knowledge graph completion~\cite{DHamaguchiOSM17}, biology~\cite{LiuYCJ20, 0059WLLZOJ22}, and recommender systems~\cite{Fan0LHZTY19, ChenWLYWY21}. 
A substantial portion of graph data originates from human society and various aspects of people's lives. Such data inherently contain sensitive information, susceptible to discrimination and prejudice.
In many high-stakes domains such as criminal justice~\cite{abs-1901-10002}, healthcare~\cite{rajkomar2018ensuring}, and financial systems~\cite{FeldmanFMSV15, abs-1909-12946}, ensuring fairness in the results of graph neural networks with respect to certain minority groups of individuals is of profound societal importance. 
These types of data are frequently distributed across multiple machines. Thus, challenges are present when applying machine learning tasks that could benefit from utilizing global data. To overcome these obstacles, federated learning has emerged as a promising solution. By leveraging a federated learning framework, organizations can collaborate on a task without compromising the confidentiality of their data. This enables the creation of a shared model while maintaining data privacy, allowing for secure and efficient cooperation among different organizations and local clients.

\nop{Graph neural networks (GNNs) have emerged as a formidable tool for generating meaningful node representations~\cite{KipfW17} and making accurate predictions on nodes by leveraging graph topology~\cite{YuanYWLJ21}.
Despite its success in graph analytics, training GNNs over centralized graph data has been limited in practice due to privacy concerns.  
Federated learning, a trending distributed learning paradigm, has emerged as a promising solution.
Several federated GNNs~\cite{XieMXY21, yao2022fedgcn} have been designed recently.}  

As a result, several works have emerged in the field of federated learning combined with GNNs~\cite{he2021fedgraphnn, XieMXY21, yao2022fedgcn} to address various challenges. He \textit{et al.}~\cite{he2021fedgraphnn} present an open benchmark system for federated GNNs that facilitates research in this area. Xie \textit{et al.}~\cite{XieMXY21} propose a graph clustered federated learning framework that dynamically identifies clusters of local systems based on GNN gradients. Yao \textit{et al.}~\cite{yao2022fedgcn} leverage federated learning to train Graph Convolutional Network (GCN) models for semi-supervised node classification on large graphs, achieving fast convergence and minimal communication. Nonetheless, there is also research indicating that predictions of GNNs (over centralized data) can be unfair and have undesirable discrimination \cite{DDongLJL22, hussain2022adversarial}.
Under the federated learning setting, the bias in the local GNN models can be easily propagated to the global model. 
While most of the prior works \cite{ezzeldin2021fairfed, DuXWT21, qi2022fairvfl} have primarily focused on debiasing federated learning models trained on non-graph data, none of them have studied federated GNNs. 
How to enhance the fairness of GNNs under federated settings remains to be largely unexplored. 

In general, the bias of GNNs can be stemmed from two different sources: (1) bias in  graph data \cite{hussain2022adversarial, KoseS22},
and (2) bias in learning algorithms (e.g., the message-passing procedure of GNNs) \cite{jiang2022fmp, DaiW21}. In this paper, we aim to enhance the group fairness of federated GNNs by mitigating both types of bias. In particular, in terms of data bias, we group the nodes by their sensitive attribute (e.g., gender and race). Then we group the links between nodes based on the sensitive attribute values of the connecting nodes: (1) the {\em intra-group links} that connect nodes belonging  to the same node groups (i.e., the same sensitive attribute value), and (2) the {\em inter-group links} that connect nodes that belong to different node groups. 
Based on the link groups, 
we consider the data bias that takes the form of the {\em imbalanced distribution between inter-group and intra-group links}, which is a key factor to the disadvantage of minority groups by GNNs \cite{hussain2022adversarial}. 
On the other hand, in terms of model fairness, we consider two well-established group fairness definitions: \textit{statistical parity} (SP) \cite{barocas2017fairness,DaiW21} and \textit{ equalized odds} (EO) \cite{jiang2022fmp, WangZDCLD22}. 

In this paper, we propose \ourmethod, one of the first federated graph neural networks that enhances group fairness of both local and global GNN models. 
Our contributions are summarized as follows: 
1) We provide theoretical insights on the connection between data bias (imbalanced distributions between inter-group and intra-group links) and model fairness (SP and EO); 
2) Based on the theoretical findings, we design \ourmethod\ that enhances group fairness of both local and global GNN models by taking both the data bias of local graphs and statistical fairness metrics of local models into consideration during the training of both local and global models; 
3) Through experiments on three real-world datasets, we demonstrate that \ourmethod\ outperforms the baseline methods in terms of both fairness and model accuracy. 
\section{Related Work}

{\bf Group fairness in GNNs.} 
Existing studies on fair graph learning can be broadly classified into three categories: data augmentations \cite{zhu2020deep, YouCSW21, luo2022automated}, problem optimization with fairness considerations \cite{KamishimaAAS12, LaclauRCL21}, and adversarial debiasing methods \cite{FisherMPC20, DaiW21, ling2023learning}. 
Data augmentations~\cite{KoseS22, wang2022uncovering} are used to create a balanced graph, followed by contrastive learning to learn fair graph representations. 
Biased edge dropout~\cite{SpinelliSHU22} is proposed to mitigate the impact of homophily and promote fairness in graph representation learning. 
Fair Message Passing (FMP) approach~\cite{jiang2022fmp} uses optimization techniques to transform the message-passing scheme into a bi-level optimization framework.
Adversarial debiasing~\cite{DaiW21} has been studied to mitigate graph bias and takes advantage of a sensitive feature estimator to require less sensitive information. 
Masrour \textit{et al.}~\cite{MasrourWYTE20} propose a post-processing step to generate more heterogeneous links to overcome the filter bubble effect and introduces a framework combining adversarial network representation learning with supervised link prediction.
Hussain \textit{et al.}~\cite{HussainCSHLSK22} inject edges between nodes with different labels and sensitive attributes to perturb the original graphs, followed by an adversarial training framework to debias the perturbed graphs. This work studies inter-group edges between different cases (e.g., different labels and different sensitive groups). However, it does not consider inter/intra-group edges' ratio.
Recent studies have utilized mixing techniques for addressing fairness concerns in graph learning. 
Wang \textit{et al.}~\cite{WangZDCLD22} generate fair views of features and adaptively clamp weights of the encoder to avoid using sensitive-related features. 
Ling \textit{et al.}~\cite{ling2023learning} use data augmentations to create fair graph representations, followed by adversarial debiasing training. However, their theoretical proofs make strong assumptions regarding the loss function being bounded by a fixed constant. All these works are focused on fairness in centralized GNNs, and none of them have considered GNNs under the federated setting.

{\bf Fairness of federated learning.} 
\nop{
Some recent studies have incorporated fairness with federated learning. 
Ezzeldin \textit{et al.}~\cite{ezzeldin2021fairfed} assign weights to local models on the client side based on their fairness scores. 
Du \textit{et al.}~\cite{DuXWT21} utilize kernel reweighing functions to assign reweighing values to each training sample in both the loss function and the fairness constraint.  
Qi \textit{et al.}~\cite{qi2022fairvfl} consider vertical federated learning and leverage adversarial training to debias learned representations. Unlike these works that consider learning models trained over non-graph data, we focus on GNNs and aim to enhance group fairness.
}
In contrast to traditional fairness considerations in machine learning, federated learning has a unique fairness concept, performance fairness, at the device level. This is akin to fair resource allocation \cite{EeB04, EryilmazS06} in wireless networks, which extends the notion of accuracy parity by evaluating the degree of uniformity in device performance. Li \textit{et al.}~\cite{LiSBS20} give a solution to address fairness concerns in federated learning through an optimization objective that minimizes an aggregated reweighted loss, where devices with higher losses are assigned proportionally higher weights. Further, Bi-level optimization schemes~ \cite{00050BS21, lin2022personalized} aim to reduce the variance of test accuracy across devices, thereby increasing federated-specific fairness. In addition to the traditional fairness performance, there is a growing body of research in fair-federated learning \cite{hu2022provably, ezzeldin2021fairfed, zeng2021improving} that focuses on the general fairness definition, such as Ezzeldin \textit{et al.}~\cite{ezzeldin2021fairfed} create a server-side aggregation step to weight local models performing well on fairness measures higher than others, and Du \textit{et al.}~\cite{DuXWT21} use kernel reweighing functions to assign reweighing values to each training sample in both the loss function and fairness constraint. Qi \textit{et al.}~\cite{qi2022fairvfl} present a fair vertical federated learning framework that assembles multiple techniques. They first learn representations from fair-insensitive data stored in some platforms and then send the learned representations to other platforms holding sensitive information. Subsequently, the authors conduct adversarial training with sensitive information to debias the representations. Unlike these works that consider learning models trained over non-graph data, we focus on GNNs and aim to enhance group fairness.
\section{Federated Graph Learning}

{\bf Notations.} Given an undirected graph $\G = (\V, \E, \X)$ that consists of a set of nodes $\V$, edges $\E$, and the node features $\X$, let $\A$ be the adjacency matrix of $\G$. 
We use $\Z$ to denote the node representations (embedding) learned by a GNN model. 
In this paper, we consider node classification as the learning task. We use $y$ and $\hat{y}$ to denote the ground truth and the predicted labels respectively.  

{\bf Federated Graph Learning.} 
In this paper, we consider \textit{horizontal federated learning} (HFL)  setting~\cite{DuSSJZZ22, PengZSBLW22} with $K$ clients $\{C_1, \dots, C_K\}$ and a server $\server$. 
Each client $C_i$ owns its private graph $\G_i$. Different $\G_i$ can have different $\V_i$ and $\A_i$. 
All clients share the same set of node features.  

In the HFL setting, clients do not share their local graphs with the server $\server$ due to various reasons (e.g., privacy protection). Thus the goal of HFL is to optimize the overall objective function while keeping private graphs locally
\begin{equation*}
\arg\min_{\omega} \sum_{i=1}^{K} \mathcal{L}_i(\omega) = \arg\min_{\omega} \sum_{i=1}^{K} \sum_{j=1}^{N_i}\loss_i(v_j, \G_i;\omega),
\end{equation*}
where $\mathcal{L}_i(\omega)$ denotes the loss function parameterized by $\omega$ on client $C_i$, $N_i$ is the number of nodes in the local graph $\G_i$, and $\loss_i(v_j, \G_i;\omega)$ denotes the average loss over the local graph on client $C_i$.

In this paper, we adapt the SOTA FedAvg framework \cite{McMahanMRHA17} to our setting. In FedAvg, only model parameters are transmitted between $\server$ and each client $C_i$. Specifically, during each round $t$, $\server$ selects a subset of clients and sends them a copy of the current global model parameters $\omega_t$ for local training. Each selected client $C_i$ updates the received copy $\omega_t^i$ by an optimizer such as stochastic gradient descent (SGD) for a variable number of iterations locally on its own data and obtains $\omega_t^i$. Then $\server$ collects the updated model parameters $\omega_t^i$ from the selected clients and aggregates them to obtain a new global model $\omega_{t+1}$. Finally, the server broadcasts the updated global model $\omega_{t+1}$ to clients for training in round $t+1$. The training process terminates when it reaches convergence.

\section{Group Fairness and Data Bias}

\label{sec: group fairness and data bias}
In this paper, we mainly focus on group fairness. 
In general, group fairness requires that the model output should not have discriminatory effects towards {\em protected groups} (e.g., female) compared with {\em un-protected groups} (e.g., male), where protected and un-protected groups are defined by a sensitive attribute (e.g., gender)~\cite{barocas2017fairness}. 

{\bf Node groups.} We adapt the conventional group fairness definition to our setting. We assume each node in the graph is associated with a \textit{sensitive attribute} $s$ (e.g., gender,  race) \cite{MehrabiMSLG21}. The nodes are  partitioned into distinct subgroups by their values of the sensitive attribute. 
For simplicity, we consider a binary sensitive attribute in this paper. In the following discussions, we  use  $s=0$ and $s=1$ to indicate the protected and un-protected node groups respectively. 


{\bf Inter-group and intra-group edges.} Based on the definition of the protected node groups, we further partition edges in a graph into two groups: 
(1) \textit{inter-group edges} that consist of the edges connecting the nodes that belong to different groups (e.g., the edges between female and male nodes); 
and (2) \textit{intra-group edges} indicate the edges connecting nodes in the same group (e.g., the edges between female and female nodes).

{\bf Data bias in edge distribution.} 
Recent studies have shown that one potential data bias in graphs that can lead to unfairness of GNNs is the imbalanced distribution between the inter-group edges and the intra-group ones in training graphs \cite{hussain2022adversarial,masrour2020bursting,spinelli2021fairdrop}.  Therefore, in this paper, we focus on one particular type of data bias which takes the form of the imbalance between the distribution of inter-group and intra-group edges. 
In particular, 
we define {\textit{ group balance score} (GBS) to measure this type of bias in edge distribution as follows.

    \begin{definition}
    \label{def:edge_balance}
        Given a graph $\G$, let $|\E_{\text{inter}}|$ and $|\E_{\text{intra}} |$ be  the number of inter-group edges  and intra-group edges  in $\G$ respectively. $|\E|= |\E_{\text{intra}}| + |\E_{\text{inter}}|$ denotes the total number of edges.
       \textit{Group balance score} (GBS) (denoted as $\EB$) is defined as follows:
        \begin{equation*}
            \EB= 1 - |\intra - \inter|,
        \end{equation*}
        where $\intra = \frac{|\E_{\text{intra}}|}{|\E|}$ and $\inter = \frac{|\E_{\text{inter}}|}{|\E|}$.
\end{definition}



Intuitively, a higher $B$ value indicates a better balance between the inter-group edges and intra-group ones. $B$ reaches its maximum (i.e., $B=1$) when the graph is perfectly  balanced between the two edge groups (i.e., $|\intra| = |\inter|$). 

{\bf Group fairness of models.}
To measure the group fairness of the target GNN model, we employ two well-established fairness definitions: \textit{ statistical parity} (SP) \cite{barocas2017fairness, KoseS22, DaiW21} and \textit{ equalized odds} (EO) \cite{HardtPNS16, jiang2022fmp, WangZDCLD22}. 
In this paper, we adapt both fairness notations to the node classification task. Formally,
\begin{definition}[Statistical Parity]
    \label{def:sp}
        Statistical parity measures the difference in probabilities of a positive outcome of node groups, which is computed as the follows:  
        \begin{equation*}
            \Delta_{\text{SP}} = |P(\hat{y}=1 | s=0) -P(\hat{y}=1 | s=1)|.
        \end{equation*}
\end{definition}
\begin{definition}[Equalized Odds]
    \label{def:eo}
    Equalized odds measures the difference between the predicted true positive rates of node groups, which is computed as the follows: 
    \begin{equation*}
        \Delta_{\text{EO}} = |P(\hat{y}=1 |y=1,  s=0) -P(\hat{y}=1 | y=1, s=1)|.
    \end{equation*}
\end{definition}
Intuitively, $\Delta_{\text{SP}}$ measures the difference between the positive rates of two different node groups, while $\Delta_{\text{EO}}$ measures the difference between the true positive rates for these two groups. Both measurements of SP and EO are equally important and complementary in terms of fairness in node classification.  We are aware of the existence of other accuracy measurements (e.g., false positive rate). These accuracy measurements will be left for future work. 

In this paper, we consider both \textit{local fairness} and \textit{global fairness}, i.e. the outputs of both local and global models should be fair in terms of statistical parity and equalized odds. As shown in the prior study  \cite{wang2023mitigating}, although local fairness and global fairness do not imply each other, global fairness is theoretically upper-bounded by local fairness for both SP and EO metrics when data follows the IID distribution. Our empirical studies (Section \ref{sc:exp}) will show that \ourmethod can achieve both local fairness and global fairness for non-IID distribution.

\section{Connection between Data Bias and Model Fairness}

\label{sec:rela_data_and_model_def}



Data bias has been shown as an important source of model unfairness. In this section, we investigate how the imbalanced distribution between the intra-group and inter-group links impacts the model fairness (in terms of SP and EO). 

Directly examining how data bias can impact SP and EO is challenging. 
Therefore, we take an alternative approach that deals with model fairness from the lens of graph representations. 
Intuitively, group fairness requires that the model outputs should not be distinguishable based on sensitive attributes. As GNNs rely on graph representations to generate the model outputs for the downstream tasks, group fairness can be achieved when the graph representations are invariant to the sensitive attributes \cite{KoseS22}.  Informally, this requires independence between the graph representation and the sensitive attribute. How can the imbalanced distribution between the intra-group and inter-group links impact the correlation between the graph representation and the sensitive attribute? Next, we present our main theoretical analysis of the impact study.

In this paper, we utilize the \textit{Point-Biserial Correlation}~\cite{glass1996statistical}, a special case of the Pearson correlation, to measure the correlation between the graph representation and the sensitive attribute. 
We do not consider other correlation metrics such as the Pearson correlation because the Point-Biserial Correlation accepts dichotomous variables as input.

\begin{definition}[Point-Biserial Correlation Coefficient]
        The point-biserial correlation coefficient is a statistical measure that quantifies the degree of association between a continuous-level variable and a dichotomous variable. In particular, given a dichotomous variable $s$ ($s\in\{0, 1\})$, and a continuous variable $X$, the point-biserial correlation coefficient $\pbc$ between  $X$ and $s$ is defined as:
            \begin{equation*}
            \pbc=\frac{\mu_0 - \mu_1}{\sigma_{X}}\sqrt{\frac{N_0 N_1}{N^2}},
            \end{equation*}
        where $\mu_0$ ($\mu_1$) is the mean value of $X$ that are  associated with $s=0$ ($s=1$), $N_0$ ($N_1$) is the number of samples in the class $s=0$ ($s=1$), $N=N_0+N_1$, and $\sigma_X$ is the standard deviation of $X$.
    \end{definition}
The conventional definition of the point-biserial correlation involves a 1-dimensional variable, $X$. In line with the existing work~\cite{KoseS22}, we extend this concept to incorporate higher-dimensional data, using a node feature matrix $\X$. The node feature matrix is denoted as $\mathbf{X} \in \mathbb{R}^{N \times d}$, where $N$ represents the number of nodes and $d$ denotes the feature dimension. In our work, we define the $\mu_0$ ($\mu_1$) stated in the above definition as the mean of $j$-th dimension of node feature of the nodes possessing a sensitive attribute $s=0$ ($s=1$), where $j \in \{1, ..., d\}$. 
The formal notions of $\mu_0$ and $\mu_1$ can be expressed as $\mu_j^{s=0}$ and $\mu_j^{s=1}$, the details of the calculation of $\mu_0$ and $\mu_1$ can be found Appendix~\ref{app:mu}. 
For convenience, we will use $\mu_0$ ($\mu_1$) and $\mu_j^{s=0}$ ($\mu_j^{s=1}$) interchangeably. After computing the means $\mu_0$ and $\mu_1$ for a specific column $j$, we can calculate the point-biserial correlation following the conventional definition. 

Given a graph $\G$, let $s$ be the sensitive attribute of $\G$ and $\Z$ be the representation (node embeddings) of $\G$. Next, we will present the theoretical analysis of the relationship between the data bias (group balance score) and the correlation $\pbcz$ between $s$ and $\Z$. 
Our theoretical analysis relies on two assumptions regarding the model and the data distribution respectively:
\begin{itemize}
\item {\bf Model assumption}: We assume the  GNN model contains a linear activation function \cite{saxe2015deep, SaxeMG13}.
We note that although our theoretical analysis relies on the assumption of linear activation function, our algorithmic techniques are generic and can be applied to various GNN models, including those with non-linear activation functions. 
\item {\bf Data assumption}:  Let $\mathbf{X}_{j}^{s=0}$ ($\mathbf{X}_{j}^{s=1}$) be the $j$-th feature of those data samples associated with the sensitive attribute $s=0$ ($s=1$). By following the prior works \cite{jiang2022fmp,ling2023learning}, we assume 
the data values of $\X_j^s$ follows a Gaussian distribution $\mathcal{N}(\mu_j^s, \sigma_j^s)$, where $\mu_j^s$ and $\sigma_j^s$ are the mean and variance of $\mathbf{X}_{j}^s$. 
\end{itemize}

Based on the assumptions, we derive the following lemma:
\begin{lemma}
    \label{lemma:rpb}
    Given a graph $\G$ and a GNN model that satisfies the data and model assumptions respectively, the Point-biserial correlation (denoted as $\pbcz$) between the graph representation  $\Z$ and the sensitive attribute $\mS$ is measured as the following:
 \begin{equation*}
        \pbcz = (\Nzero \mu_0 - \None \mu_1)(\intra -\inter)\cdot\frac{\sqrt{\Nzero\,\None}}{\SDZ\Ntotal^2},
 \end{equation*}
where $\intra$ and $\inter$ are defined by Definition \ref{def:edge_balance}, $\Nzero$ $(\None)$ is the number of nodes associated with $s=0$ $(s=1)$, $\SDZ$ is the standard deviation of $\Z$, $\Ntotal$ is the total number of nodes of $\G$, and $\mu_0$ ($\mu_1$) is the mean value of nodes that are associated with $s=0$ ($s=1$).
\end{lemma}

The proof of Lemma \ref{lemma:rpb} can be found in Appendix~\ref{proof:lemma}. Intuitively, when $\pbcz=0$ (i.e., the graph representation is independent of the sensitive attribute), either $\Nzero \mu_0 = \None \mu_1$ or $\intra = \inter$. Remark that both $\Nzero$ and $\None$ are fixed values for a given graph. Therefore, we derive the following important theorem: 

\begin{theorem}
\label{theorm:rpb}
Let $\G$ be a graph that satisfies the given data assumption, and $\Z$ be the representation of $\G$ learned by a GNN model that satisfies the given model assumption. 
Then the Point-biserial correlation $\pbcz$ between the graph representation $\Z$ and the sensitive attribute $\mS$ is positively correlated/monotonically increasing with $|\intra - \inter|$. In particular, $\pbcz = 0$ when $\intra = \inter$. 
\end{theorem}
The details of the proof for Theorem~\ref{theorm:rpb} can be found in Appendix~\ref{proof:theorem}.
As Theorem~\ref{theorm:rpb} suggests, the imbalanced distribution between inter-group and intra-group links is one of the key factors to the correlation between the graph representation and the sensitive attribute. 
As the model predictions largely rely on the graph representations, the unfairness in the model predictions can be largely reduced if the graph representation is independent of the sensitive attribute. Therefore, a graph with a more (less, resp.) balanced distribution between inter-group and intra-group links  should lead to a more (less, resp.) fair model. We will follow this insight and design \ourmethod.

\section{The Proposed Method: \ourmethod}
\label{sec:framework}

\begin{figure}[t]
      \centering
      \includegraphics[width=1\textwidth]{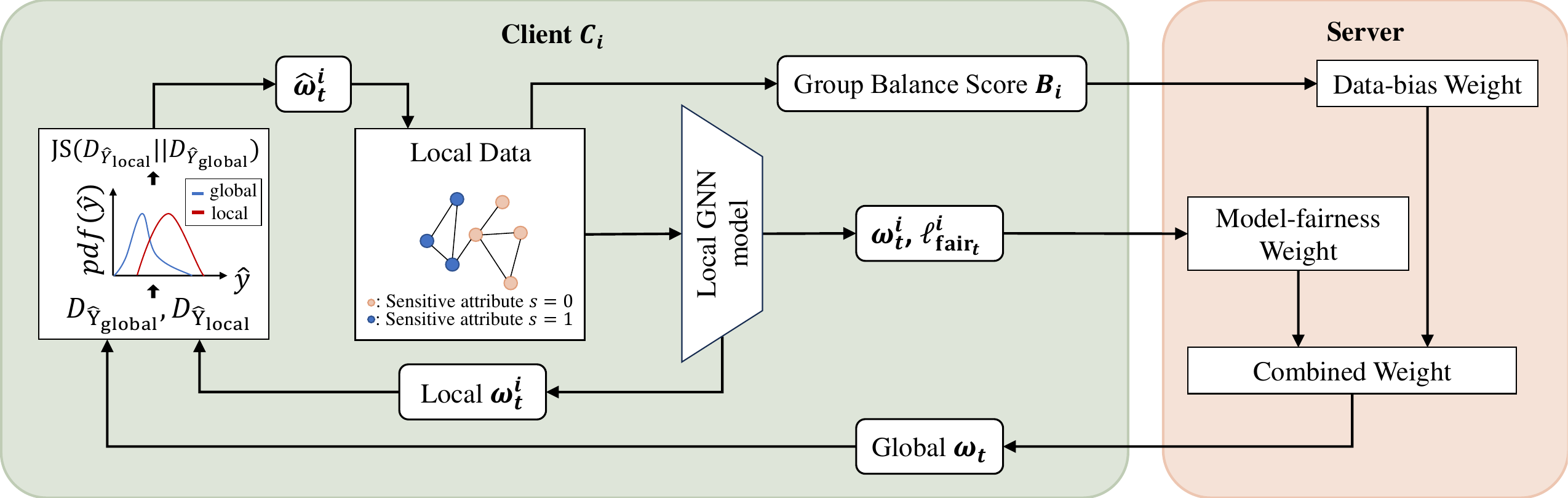}
      \caption{\small An overview of \ourmethod. In each iteration, client $C_i$ receives the global model, calculates the Jensen-Shannon (JS) divergence $\jsold$, and updates the local model $\hat{\omega}^i_t$. Local data is processed through a 2-layer GCNs, and the model is trained over several epochs. The updated model $\omega^i_t$ and the group balance score $B_i$ are uploaded to the server. The server then computes the data-bias and model-fairness weights, integrates them, and uses the combined weight to aggregate the local models and update the global model $\omega_t$. 
      }
      \label{fig:overview}
\end{figure}

In a nutshell, \ourmethod\ follows the same principles as FedAvg \cite{McMahanMRHA17}: in each round $t$, the server $\server$ uniformly samples a set of clients. Each selected client $C_i$ receives a copy of the current global parameters $\omega_{t}$, and obtains its local model update $\omega_{t}^i$ by learning over its local graph $\G_i$ with $\omega_{t}$ as the initialized model parameters. Then each client $C_i$ sends $\omega_{t}^i$ back to $\server$. $\server$ aggregates all the received local model updates on $\omega_{t}$ and obtain $\omega_{t+1}$. To equip both local fairness and global fairness with the federated learning framework, 
{\ourmethod} comprises two main components: 
\begin{itemize}
\item \textit{Client-side fairness-aware local model update scheme}: each client $C_i$ obtains $\omega_{t}^i$ by adding local fairness as a penalty term to the loss function of its local model. 
\item \textit{Server-side fairness-weighted global model update scheme}: When $\server$ aggregates all local model updates from the clients, it takes both the fairness metrics of local models (e.g., SP or EO) and data bias of local graphs as aggregation weights. 
\end{itemize}
Figure~\ref{fig:overview} illustrates the framework of \ourmethod, and detailed steps of our algorithm are summarized in Algorithm~\ref{alg:alg1}. Next, we present the details of each component. 


\subsection{Client-side Fairness-aware Local Model Update}

Intuitively, the design of the local model update scheme should serve two objectives: 1) it ensures the convergence of not only the local models on the client side but also the global model on the server side, especially when the data across different clients is non-iid, and 2) it equips fairness (both SP and EO) with the local models by seamlessly incorporating the fairness constraint into the training process of local model updates. To achieve these objectives, we utilize two key techniques: (i) model interpolation by leveraging the Jensen-Shannon (JS) divergence~\cite{Nielsen21}, and (ii) adding  fairness as a penalty term to the loss function of the local models. Next, we discuss the details of each technique.
 

{\bf JS  divergence between global and local models.} 
One challenge of GNN training under federated setting is that the local data on each client may not be IID \cite{zhao2018federated, LiHYWZ20}. The difference in data distribution can lead to a contradiction between minimizing local empirical loss and reducing global empirical loss and thus prevent the global model  to reach convergence. To tackle the non-IID challenge, inspired by Zheng \textit{et al.}~\cite{ZhengZCWWZ21}, \ourmethod\ measures the difference between the local and the global model, and takes such difference into consideration when updating local model parameters. 


Specifically, let $\omega_{t}$ and $\omega^i_{t-1}$ be the global model parameters at epoch $t$ and the local model parameters maintained by the client $C_i$ at epoch $t-1$, respectively. 
Let $D_{\hat{\Y}_{\text{global}}}$ and $D_{\hat{\Y}_{\text{local}}}$ be the distribution of the  node labels in client $C_i$'s local graph predicted by the global model (with $\omega_{t}$ as parameters) and local model (with $\omega^i_{t-1}$  as parameters). 
Upon receiving  $\omega_{t}$, the client $C_i$ calculates the Jensen-Shannon (JS) divergence  (denoted as $\jsold$) between $D_{\hat{\Y}_{\text{global}}}$ and $D_{\hat{\Y}_{\text{local}}}$ as follows:
\begin{equation*}
\jsold=\mathrm{JS}(D_{\hat{\Y}_{\text{global}}}||D_{\hat{\Y}_{\text{local}}}).
\end{equation*}

Intuitively, a higher $\jsold$ value indicates a greater difference between the global model and the local one. Next, the client $C_i$ computes the initialized value of its local model update $\omega^i_t$ (denoted as $\hat\omega^i_t$) as follows: 
\begin{equation}
\label{eqn:client-local-model-update}
    \hat\omega^i_t \leftarrow (1-\jsold) \cdot \omega^i_{t-1} + \jsold \cdot \omega_{t}.
\end{equation}
Following Eqn.~\eqref{eqn:client-local-model-update},  when the global model largely varies from the local model, the global model $\omega_{t}$ is considered more important than the local model $\omega^i_t$. Consequently, $\omega^i_t$ will be closer to $\omega_{t}$.  This will expedite the global model to move towards the global stationary point and reach convergence. 

{\bf Fairness penalty term.} To equip fairness with the local models, 
we focus on a widely-used in-process approach, where fairness criteria are introduced during the training process, typically by adding a fairness constraint or penalty to the formulation of their objective function \cite{ZengIKFSP21, YaoH17}. 
Therefore, we do not include sensitive features as inputs to the GNNs but instead use them during training as part of a penalty on disparity. Formally, we define the fairness penalty term of the client $C_i$ at epoch $t$  (denoted as $\loss^i_{\text{fair}_t}$) as the sum of two types the norm-1 of group fairness (SP and EO) measured over $C_i$'s local graph $\G_i$: 
\begin{equation*}
\loss^i_{\text{fair}_t} = ||{\Delta^i_{\text{SP}_t}}||_1 + ||{\Delta^i_{\text{EO}_t}}||_1.
\end{equation*}

Then, the fairness penalty term is added to the loss function of the local model as follows: 
\begin{equation}
\label{eqn:lossfunction}
\loss(\omega^i_t;\Z_i) = \loss^i_{\text{util}_t} + \alpha\,\loss^i_{\text{fair}_t},
\end{equation}
where $\loss^i_{\text{util}_t}$ denotes the utility loss (e.g., entropy loss for node classification), and $\alpha$ is a tunable hyper-parameter.

{\bf Local model parameter updates.} Finally, the client $C_i$ obtains the final local model update $\omega^i_t$ by executing multiple SGD steps: 
\begin{equation*}
\omega^i_t \leftarrow \hat\omega^i_t - \eta \cdot \nabla \loss(\omega^i_t;\Z_i),
\end{equation*}
where $\hat\omega^i_t$ is the initialized local model update (Eqn.~\eqref{eqn:client-local-model-update}), $\eta$ is the learning rate, and $\nabla \loss(\omega^i_t;\Z_i)$ denotes the gradient of the loss function (Eqn.~\eqref{eqn:lossfunction}).

\nop{
\begin{align}
    \begin{cases}
        \loss^i_{\text{total}_t} = \loss^i_{\text{fair}_t} + \alpha\,\loss^i_{\text{fair}_t} \\
        \omega^i_t \leftarrow \omega'^i_t - \eta \cdot \nabla \loss^i_{\text{total}_t}\\
    \end{cases}
\end{align}
\Wendy{In $\loss^i_{\text{total}_t}$, why both terms are $\loss^i_{\text{fair}_t}$?}
}

After completing local training, the client $C_i$ uploads its local model parameters $\omega^i_t$, the local fairness loss ${\loss^i_{\text{fair}_t}}$, and a group balance score $\EB^i$ to the server $\server$. We remark that sharing aggregated information of the local graphs such as ${\loss^i_{\text{fair}_t}}$ and $\EB^i$ with the server will make attacking individual information harder and thus protect their privacy.


\subsection{Server-side Fairness-weighted Global Model Update}
\label{subsec:server-aggregation}

\nop{
Given a server aggregation scheme $\mathrm{AGG}(\cdot)$, we formulate the global objective function as a bi-level problem:
\begin{equation*}
\underset{\omega \in \mathbb{R}^d}{\min} F(\omega):=\mathrm{AGG}\left(F_1(\omega),...,F_i(\omega)\right)
\end{equation*}
For each client $k \in {1, ..., K}$, $F_i(\omega)= \mathbb{E}_{\omega}[\loss(\omega;\Z_i)]$, where $\Z_i$ is the output of the local Graph Convolutional Networks (GCNs) used for node classification.
}
The server-side model update scheme in each iteration consists of four steps: 
\begin{itemize} 
\item \textit{Step 1.} Calculate the \textit{data-bias weight} ${\gamma_\E}_t$ based on the group balance scores provided by each client;
\item \textit{Step 2.} Derive the \textit{model-fairness weight} ${\gamma_\F}_t$ using the fairness loss metrics uploaded by the clients, ensuring fairness metrics are integrated into the global model aggregation process; 
\item \textit{Step 3.} Combine the data-bias weight and model-fairness weight into a combined weight ($\gamma_t$); and 
\item \textit{Step 4}. Aggregate the local model updates with the global model utilizing the combined weight. 
\end{itemize}
The following elaboration provides a detailed description of each step. 

{\bf Data-bias weight.} Upon obtaining local parameters from the selected subset of clients at epoch $t-1$, the server first constructs a data-bias weight ${\gamma_\E}_t \in \mathbb{R}^{K'}$, where $K'$ represents the number of clients in the subset: 
\begin{equation*}
{\gamma_\E}_t = \operatorname{softmax} \big(\EB^1, ..., \EB^{K'}\big). 
\end{equation*}
Each element in ${\gamma_\E}_t$ corresponds to the group balance score $\EB^i$ for each client $C_i$. Given that each client $C_i$ maintains a graph $\G_i$, the data-bias weight ${\gamma_\E}_t$ effectively gauges the equilibrium between inter- and intra-edges within each client graph.

{\bf Model-fairness weight.} While the group balance score provides a useful measure of the balance between inter- and intra-edges within each client graph, it is static and unable to capture the model's evolving learning process regarding fairness. To address this, we also consider a dynamic \textit{fairness metric weight} $\gamma_\F \in \mathbb{R}^{K'}$ that evaluates the statistical parity ${\Delta^i_{\text{SP}}}_t$ and equalized odds ${\Delta^i_{\text{EO}}}_t$ of each client's local model at each iteration:
\begin{equation*}
{\gamma_\F}_t =\exp\big(\operatorname{softmax}({\gammaFdot}_t)\big).
\end{equation*}
To calculate ${\gamma_\F}_t$, we start by a weight vector ${\gammaFdot}_t$, which takes each client's sum of two types of group fairness metrics as the element. Therefore, ${\gammaFdot}_t$ can be illustrated as:
\begin{equation*}
{\gammaFdot}_t = \big[{\Delta^1_{\text{SP}_t}} + {\Delta^1_{\text{EO}_t}}, ..., {\Delta^{K'}_{\text{SP}_t}} + {\Delta^{K'}_{\text{EO}_t}}\big].
\end{equation*}

To magnify the effect of this dynamic model-fairness weight ${\gammaFdot}_t$, we use an exponential function to rescale it after passing through a softmax function. This expands its range while maintaining its relative proportions as exponential functions grow faster than linear ones. 

We employ the softmax function to standardize the two weight vectors, making them comparable. Given the distinct properties of the group balance score $\EB$, which is bounded within the interval $[0, 1]$, each element in ${\gamma_\E}_t$ adheres to the same constraints. On the other hand, each element in ${\gamma'_\F}_t$ is a composite of two statistical fairness metrics for each client. This combination yields a value bounded within the range of $[0, 2]$. By utilizing the $\operatorname{softmax}$ function, we standardize the ranges of these two weights, harmonizing their scales and facilitating subsequent hyperparameter tuning and combination operations.

{\bf Combined weight.} After computing ${\gamma_\E}_t$ and ${\gamma_\F}_t$, we integrate them into a singular weight vector $\gamma_t \in \mathbb{R}^{K'}$ using the following equation:
\begin{equation}
\label{eqn:combined}
\gamma_t = \operatorname{softmax}\big(\frac{\lambda \cdot {\gamma_\E}_t + {\gamma_\F}_t}{\tau}\big),
\end{equation}
where $\lambda$ is a hyperparameter, and $\tau$ is the temperature parameter of the $\operatorname{softmax}$ function. This equation normalizes the two weight vectors ${\gamma_\E}_t$ and ${\gamma_\F}_t$ to probability distributions that sum to one, thereby producing a new weight vector $\gamma_t$. The temperature parameter controls the smoothness of the probability distribution, while the hyperparameter $\lambda$ adjusts the relative importance of data bias  in the final combination.

{\bf Global model updating strategy.} With the unified weights, we update the global model $\omega_t$ by combining the uploaded local models $\omega^i_t$ from $K'$ clients, where $i \in {1, ..., K'}$, with the combined weight vector $\gamma_t$ as follows:
\begin{equation}
\label{eq:global update}
\omega_t \leftarrow \gamma^1_t \cdot \omega^1_t + ...+ \gamma^{K'}_t \cdot \omega^{K'}_t,
\end{equation}
where $\gamma^i_t$ is the $k$-th element of the combined weight vector. The updated global model $\omega_t$ is then broadcasted to all clients by the server. 

\section{Experimental Evaluation}

\label{sc:exp}

\nop{

{\bf To-do items}

\begin{itemize}
    \item Finish all results on NBA dataset.
    \item Sample graphs of different distribution of group balance scores across the clients (e.g., uniform distribution of group balance scores vs. skewed distribution of group balance scores). 
    The uniform setting should have 2 settings: (1) all clients have balanced graph (large balance scores; 
    and (2) all clients have imbalanced graphs where these graphs have similar balance scores that are large. 
    The skewed setting should have some graphs of small balance score but the others have large ones. 
    Measure both fairness and accuracy for these different settings. Compare the performance under these settings. 
    \item Add the results of various $\alpha$ values (for local model setup) too. All the hyper-parameters you have tried are for global model setup. You don't need to worry about comparing with the baselines for this part of experiments.
\end{itemize}

}

In this section, we present the results of our experiments on three real-world datasets for node classification tasks. Our proposed approach, \ourmethod\; is evaluated and compared against several baseline schemes in terms of both utility and fairness metrics.

\subsection{Experimental  Setup}

{\bf Datasets.} We employ three real-world datasets, Pokec-z, Pokec-n, and NBA, that are widely used for GNN training \cite{DaiW21, jiang2022fmp, DDongLJL22, FrancoNDAO22, KoseS22, OnetoND20}. 
Both  Pokec-z and Pokec-n graphs are social network data collected from two regions in Slovakia \cite{takac2012data}. We pick the  {\tt region} attribute as the sensitive attribute and {\tt working fields} as the label for both Pokec-z and Pokec-n graphs.  
The NBA dataset comprises data on nearly 400 basketball players, including their age, nationality, and salary. The edges in this dataset represent the Twitter connections among the players. We consider the {\tt nationality} attribute as the sensitive attribute  and {\tt salary} as the prediction label (whether the salary is higher than the median). 
We binarize the sensitive attributes and prediction labels for all the three datasets. 
The statistics of these datasets are summarized in Appendix~\ref{app:data}. 

{\bf Evaluation metrics.} 
In terms of GNN performance, we measure {\em Accuracy} and {\em AUC score} for node classification.
Regarding fairness, we measure  $\Delta_\text{EO}$ and $\Delta_\text{SP}$ 
of the predictions. Intuitively, smaller $\Delta_\text{EO}$ and $\Delta_\text{SP}$  indicate better fairness. 
Furthermore, we measure two kinds of trade-offs between fairness and model accuracy as follows:
\begin{align*}
    &\operatorname{Trade-off}_{\text{ACC}} = \operatorname{Accuracy} / \left(\Delta_\text{EO} + \Delta_\text{SP}\right),\\
    &\operatorname{Trade-off}_{\text{AUC}} = \operatorname{AUC} / \left(\Delta_\text{EO} + \Delta_\text{SP}\right).
\end{align*}
    
Intuitively, a model that delivers higher accuracy and lower EO and SP values (i.e., lower unfairness) will lead to a better trade-off.

{\bf Baselines.} 
To the best of our knowledge, there is no existing work on fairness issues in federated GNNs. Thus we adapt existing methods of fair federated learning and fair GNNs. We consider two baseline methods specifically: 
\begin{itemize}
    \item We deploy GNNs under an existing fairness-enhancing federated learning framework named  {\bf FairFed}~\cite{ezzeldin2021fairfed}. 
    \item We deploy an existing fairness-aware GNN model named {\bf FairAug}~\cite{KoseS22} to the federated framework (denoted as FairAug$_{\text{+FL}}$). In particular, we replace the local models with FairAug.  
\end{itemize}

 To ensure a fair comparison between different models on each dataset, we remove all the isolated nodes in the datasets. 

{\bf GNN Setup.} We deploy a 2-layer graph convolutional networks (GCNs), 
tailoring neuron configurations to each dataset. Specifically, the GCNs are set up with 64 neurons for the NBA and Pokec-n datasets, whereas 128 neurons per layer for the Pokec-z dataset. We utilize the Rectified Linear Unit (ReLU)~\cite{fukushima1969visual} and the Adam optimizer~\cite{KingmaB14}. Our \ourmethod\, implementation is developed using the Deep Graph Library (DGL)\cite{wang2019dgl}, PyTorch Geometric\cite{Fey19}, and NetworkX~\cite{SciPyProceedings_11}. 
The details of GCNs are described in Appendix~\ref{app:exp_detail}. 

{\bf Federated learning setup.} In the federated learning setting, we randomly select a set of nodes from each graph data set to serve as the centers of egocentric networks, which also determine the number of clients, and we specify a certain number of hops to grow the networks. Each network (or subgraph) is considered as a client.
We evaluate our proposed approach on the Pokec datasets in a setting with 30 clients, each holding a 3-hop ego-network, and on the NBA dataset in a setting with 10 clients, each holding a 3-hop ego-network. See more details in Appendix~\ref{app:exp_detail}.



\begin{table}[t]
\caption{Node classification and fairness evaluation on the test sets of the \textbf{Pokec-z} and \textbf{Pokec-n} datasets. 
}

\label{tab:results_pokec}
\centering
\setlength{\tabcolsep}{0.8pt}  
\resizebox{1\textwidth}{!}{
\begin{tabular}{@{}lcccccccc@{}}
\toprule
  & \multicolumn{8}{c}{{\bf Global}} \\
  \cmidrule(lr){2-9}
  \multirow{ 2}{*}{Method} & \multicolumn{4}{c}{Pokec-z} & \multicolumn{4}{c}{Pokec-n}   \\
\cmidrule(lr){2-5} \cmidrule(lr){6-9}
   & Accuracy(\%) $\uparrow$ & AUC(\%) $\uparrow$   & $\Delta_{\text{SP}}$ (\%) $\downarrow$ & $\Delta_{\text{EO}}$(\%) $\downarrow$  & Accuracy(\%) $\uparrow$ & AUC(\%) $\uparrow$   & $\Delta_{\text{SP}}$ (\%) $\downarrow$ & $\Delta_{\text{EO}}$(\%) $\downarrow$ \\ \cmidrule(lr){1-5} \cmidrule(lr){6-9}
FairAug$_{\text{+FL}}$ & $64.62\pm0.65$              & $64.73\pm0.67$ & $4.01\pm0.28$                 & $3.98\pm0.51$                                                & $64.38\pm0.38$              & $62.45\pm0.47$ & $4.31\pm1.32$                 & $5.98\pm1.96$                                            \\
FairFed    & $65.48\pm2.63$              & $69.72\pm2.7$  & $2.92\pm1.28$                 & $3.08\pm1.21$                                                & $61.16\pm2.46$               & $61.97\pm1.93$ & $2.66\pm1.20$                 & $3.69\pm2.95$                                                \\
\ourmethod  & $\mathbf{68.17\pm0.19}$              & $\mathbf{73.47\pm0.52}$ & $\mathbf{1.66\pm1.02}$                 & $\mathbf{1.49\pm0.52}$                                                & $\mathbf{67.00\pm0.27}$               & $\mathbf{69.62\pm1.34}$  & $\mathbf{0.85\pm0.31}$                 & $\mathbf{1.00\pm1.03}$  \\
\midrule
& \multicolumn{8}{c}{{\bf Local}} \\
\cmidrule(lr){2-9}
Method
& Local Acc(\%) $\uparrow$ & Local AUC(\%) $\uparrow$   & Local $\Delta_{\text{SP}}$ (\%) $\downarrow$ & Local $\Delta_{\text{EO}}$(\%) $\downarrow$  & Local Acc(\%) $\uparrow$ & Local AUC(\%) $\uparrow$   & Local $\Delta_{\text{SP}}$ (\%) $\downarrow$ & Local $\Delta_{\text{EO}}$(\%) $\downarrow$ \\ \cmidrule(lr){1-5} \cmidrule(lr){6-9}
FairAug$_{\text{+FL}}$ & $63.33\pm4.91$              & $57.8\pm3.76$ & $8.54\pm0.53$ & $8.54\pm0.53$                                               & $61.03\pm5.89$              & $52.50\pm 4.23$ & $21.12\pm0.65$                 & $28.11\pm0.71$                                            \\
FairFed    & $59.44\pm3.60$              & $63.7\pm 12.13$ & $8.12\pm3.32$                 & $7.69\pm1.96$                                                & $56.26\pm2.42$               & $58.41\pm1.75$ & $8.38\pm4.93$                 & $7.79\pm4.36$                                                \\
\ourmethod  & $\mathbf{68.43\pm1.99}$              & $\mathbf{75.74\pm2.13}$ & $\mathbf{7.24\pm0.99}$                 & $\mathbf{7.35\pm1.58}$                                                & $\mathbf{68.19\pm2.40}$               & $\mathbf{69.49\pm3.81}$  & $\mathbf{8.38\pm2.27}$                 & $\mathbf{7.23\pm5.11}$  \\
\bottomrule
\end{tabular}
}
\end{table}


\subsection{Performance Evaluation}

Table~\ref{tab:results_pokec}  
provide a comprehensive performance comparison of various models on the Pokec datasets, with a particular focus on their fairness attributes, specifically $\Delta_{\text{SP}}$ and $\Delta_{\text{EO}}$. The results on the NBA dataset can be found in Appendix~\ref{app:etra_exp} Table~\ref{tab:results_NBA}.

 Specifically, we evaluate the models on a global test set as well as local test sets. Each local client has its corresponding test set. After we get the metrics from all clients, we report the median value of these clients for each metric. Among the evaluated methods, our proposed framework, \ourmethod\; consistently outperforms other methods across all three datasets in both global and local sets. In particular, \ourmethod\ achieves the lowest $\Delta_{\text{SP}}$ and $\Delta_{\text{EO}}$ values in all three datasets, indicating the effectiveness of the proposed scheme for group fairness in federated settings.

The trade-off performance of each model on the three datasets is illustrated in Table~\ref{tab:trade-off}. 
The proposed method achieves the highest trade-off values for the Pokec-z, Pokec-n, and NBA datasets, respectively, demonstrating its exceptional ability to balance accuracy and fairness.



\begin{table}[t]
\caption{Trade-off evaluation of \ourmethod\ and baselines.}
\label{tab:trade-off}
\centering
\resizebox{0.82\textwidth}{!}{
\setlength{\tabcolsep}{2.6pt}
\begin{tabular}{lccccccccc}
\toprule
&\multicolumn{3}{c}{Pokec-z} & \multicolumn{3}{c}{Pokec-n} & \multicolumn{3}{c}{NBA} \\
\cmidrule(lr){2-4} \cmidrule(lr){5-7} \cmidrule(lr){8-10}
Metric & \ourmethod & FairFed & FairAug$_{\text{+FL}}$ & \ourmethod & FairFed & FairAug$_{\text{+FL}}$ & \ourmethod & FairFed & FairAug$_{\text{+FL}}$\\
\midrule
$\operatorname{Trade-off}_{\text{ACC}}$ $\uparrow$ & $21.6445$ & $5.6581$ & 11.129 & $36.2460$ & $5.6581$ & 8.2651 & $13.4773$ & $5.6581$ & 2.4302\\
$\operatorname{Trade-off}_{\text{AUC}}$ $\uparrow$ & $23.3264$ & $5.6598$ & 12.0854 & $37.6649$ & $12.5286$ & 8.235 & $14.3269$ & $1.8336$ & 2.7854\\
\bottomrule
\end{tabular}
}
\end{table}

In summary, our proposed \ourmethod\;framework outperforms other baselines in terms of $\Delta_{\text{SP}}$, $\Delta_{\text{EO}}$, and trade-off values across all datasets. These results validate the effectiveness and robustness of \ourmethod\;in achieving competitive performance while maintaining group fairness in node classification tasks. These findings also highlight the broad applicability of our approach in the field of federated learning, suggesting its potential for adaptation to other domains.

Given the limited number of baseline models, we further evaluate our approach in comparison to several Fair GNNs within a non-federated setting. Comprehensive experimental results and supplementary information can be found in Table~\ref{tab:results_non_fed} presented in Appendix~\ref{app:etra_exp}.

\subsection{Ablation Study}
In the ablation study reported in Figure~\ref{fig:ablation}, we evaluate the significance of the client-side fairness-aware local model updates and the server-side fairness-weighted model updates in our proposed model on the Pokec datasets and the NBA dataset. The omission of the client-side updates or server-side updates leads to a noticeable decrease in accuracy and an increase in the disparity measures $\Delta_{\text{SP}}$ and $\Delta_{\text{EO}}$, underscoring their importance in maintaining both classification accuracy and fairness. Although discarding the server-side fair aggregation scheme results in the highest accuracy and AUC values on the Pokec-z dataset, it negatively impacts fairness, indicating its significant role in promoting fairness. The full model, $\text{F}^2$GNN achieves the best balance between accuracy and fairness, demonstrating the effectiveness of all components in federated graph learning.


\begin{figure}[t]
    \centering    
    {\includegraphics[width=0.35\textwidth]{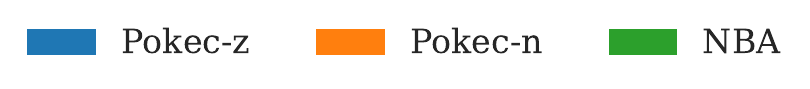}}\\
    \subfigure[Global Model Accuracy]
    {\includegraphics[width=0.30\textwidth]{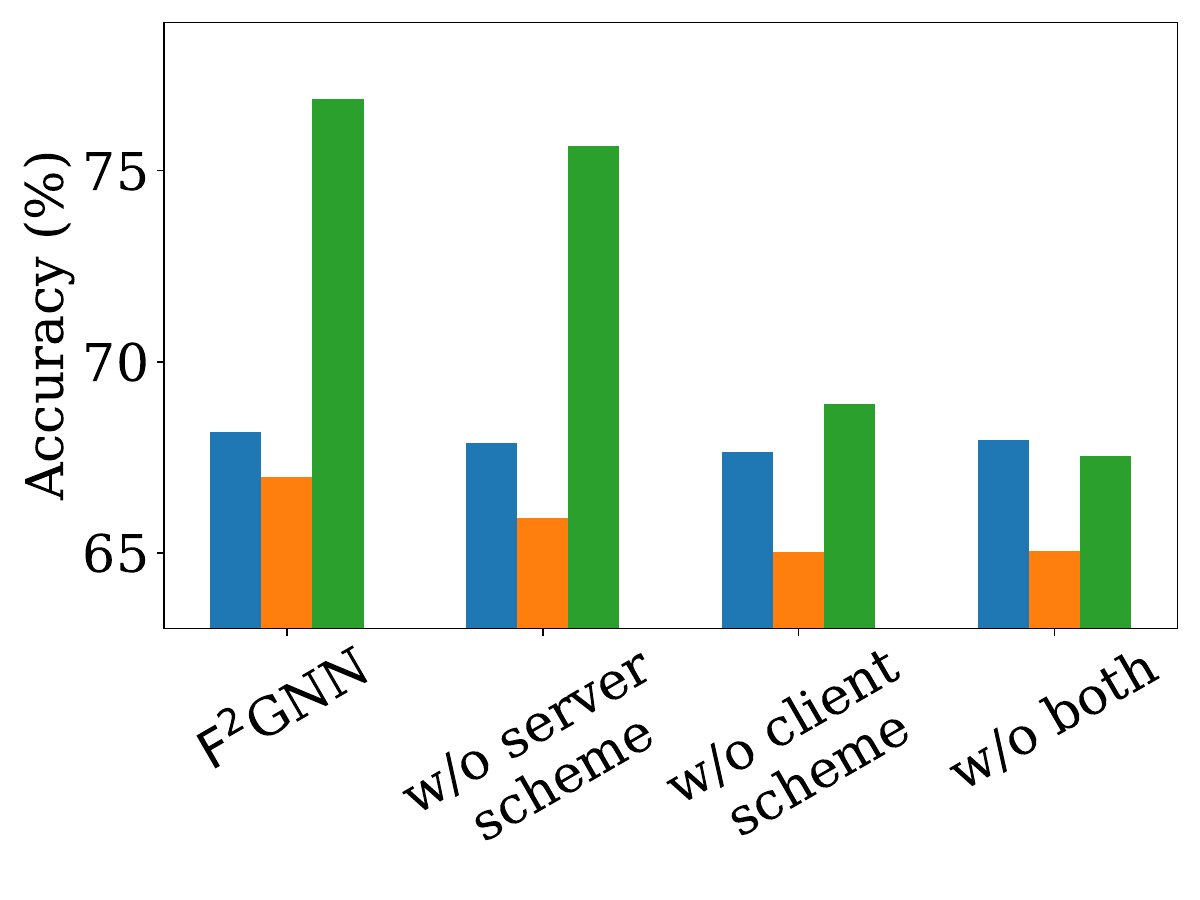}}
    \subfigure[Global Model AUC]
    {\includegraphics[width=0.30\textwidth]{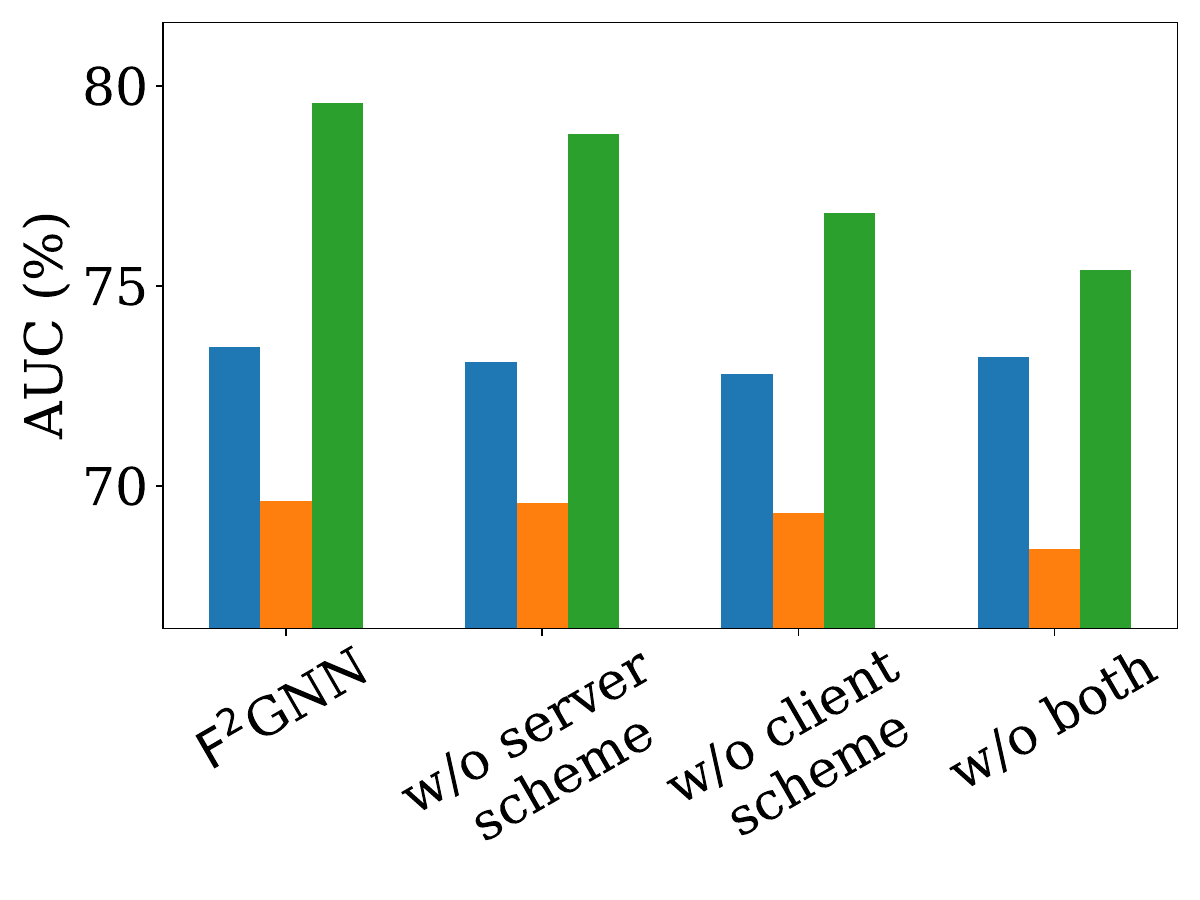}}
    \subfigure[Global $\Delta_{\text{SP}}$]{\includegraphics[width=0.30\textwidth]{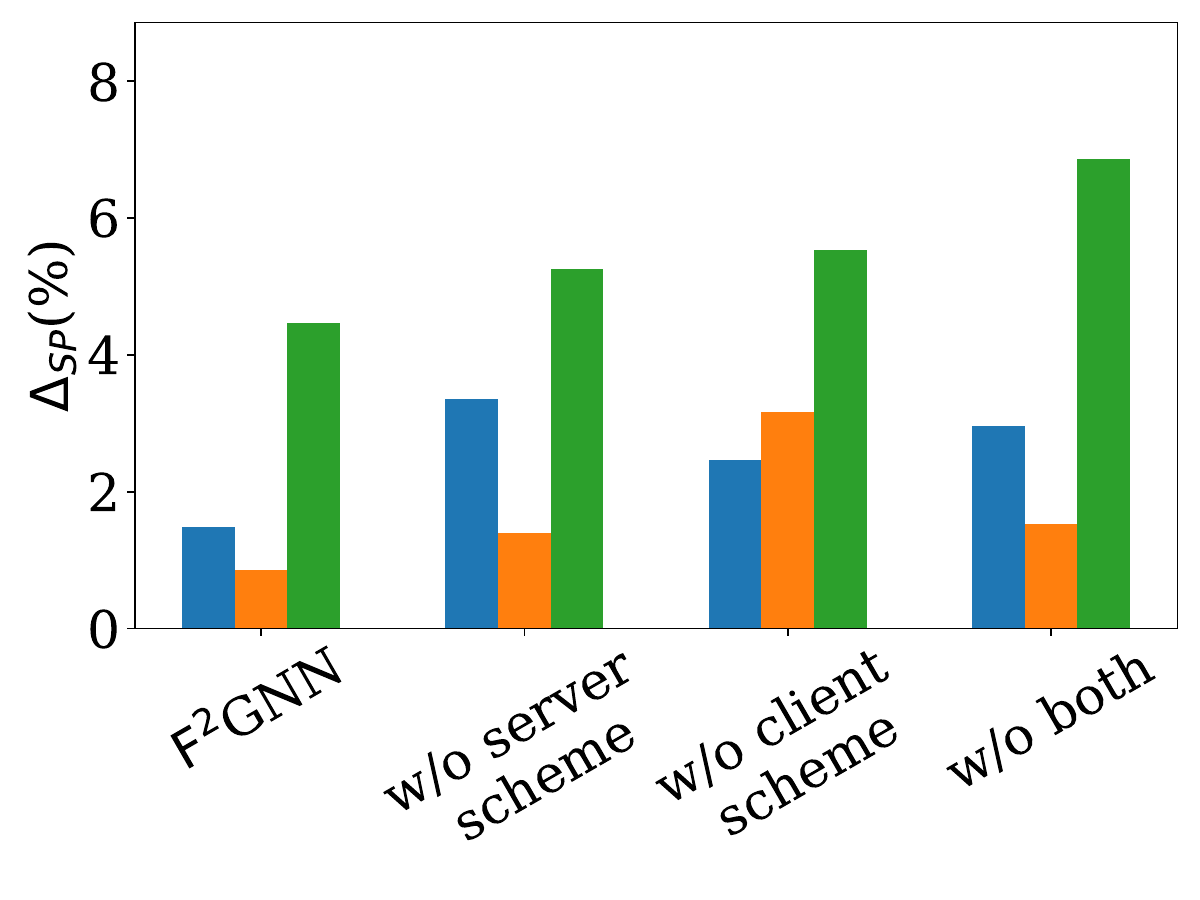}}
    \subfigure[Global $\Delta_{\text{EO}}$]{\includegraphics[width=0.30\textwidth]{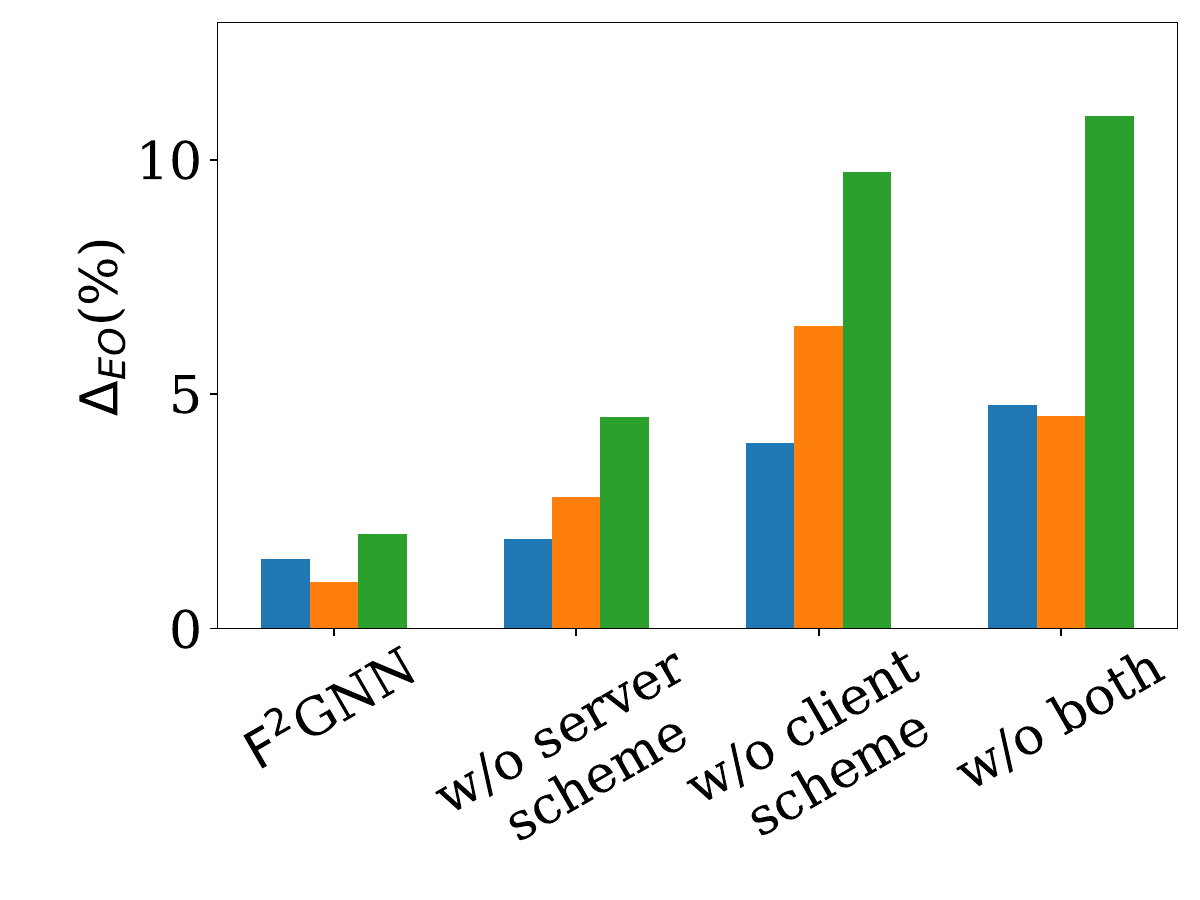}}
    \subfigure[Local $\Delta_{\text{SP}}$]{\includegraphics[width=0.30\textwidth]{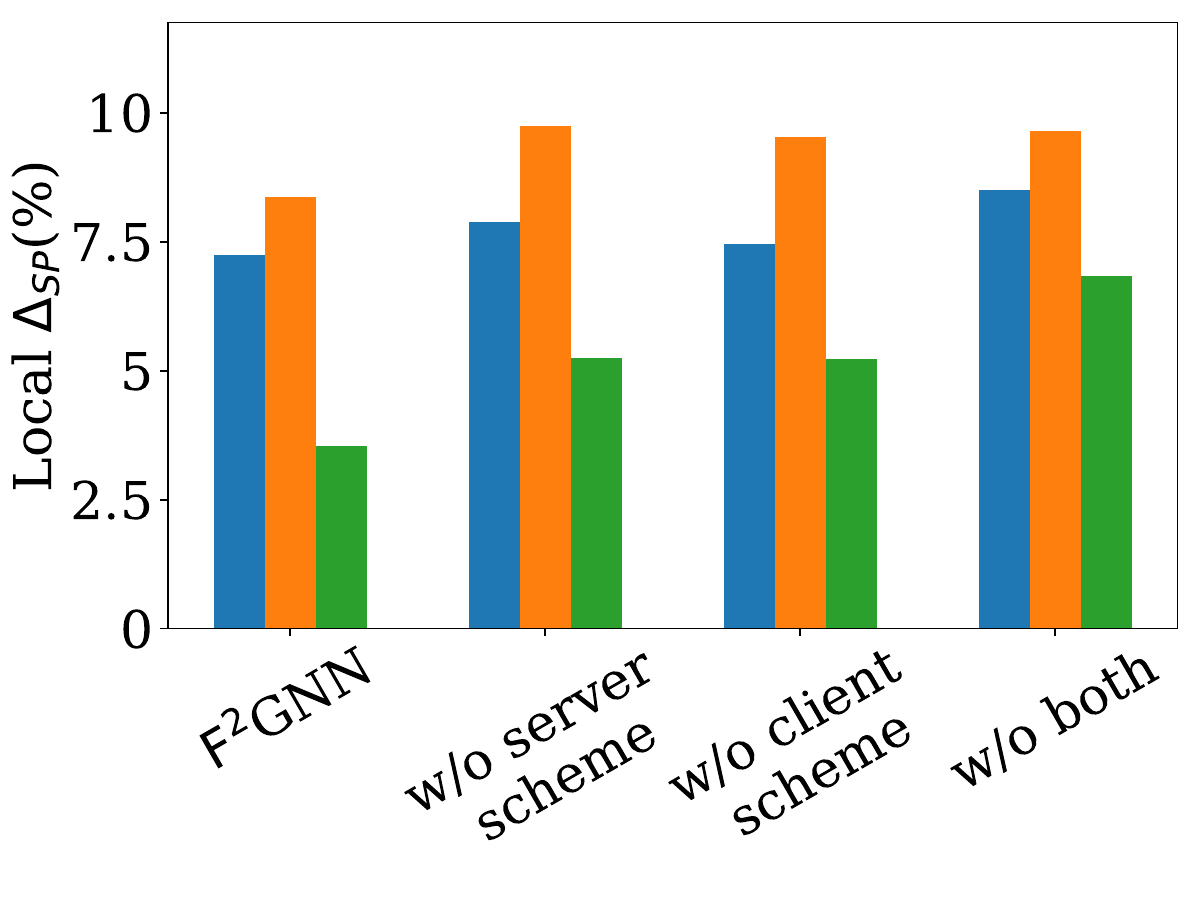}}
    \subfigure[Local $\Delta_{\text{EO}}$]{\includegraphics[width=0.30\textwidth]{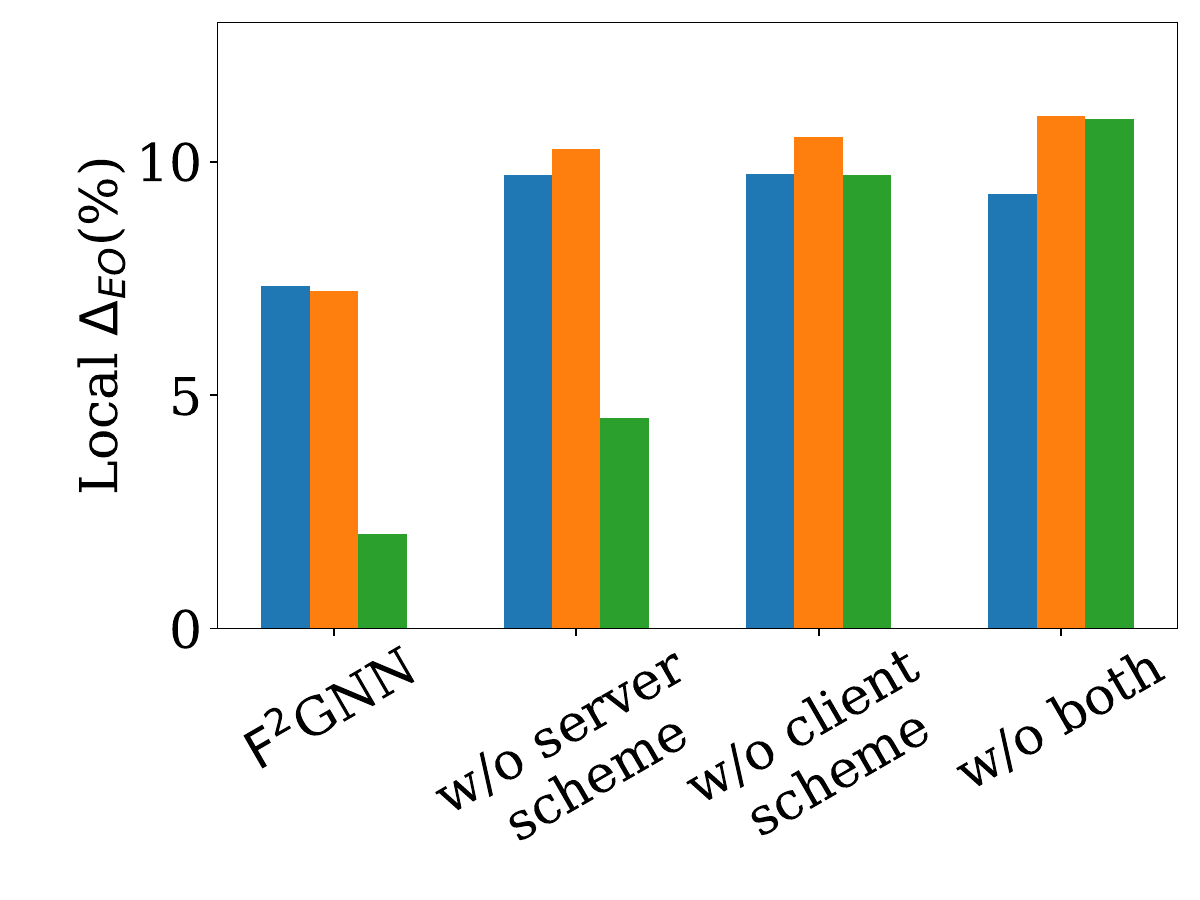}}
    \caption{Ablation study of the proposed method on three datasets.} 
    \label{fig:ablation}
\end{figure}

\section{Conclusion}


This research investigates the impact of data bias in federated learning, particularly in the context of graph neural networks. We emphasize the significance of considering the correlation between sensitive attributes and graph structures, in achieving model fairness. To address these challenges, we propose \ourmethod, a comprehensive approach that incorporates fairness in both client-side local models and server-side aggregation schemes. Our experiments on real-world datasets demonstrate the effectiveness of our algorithm in mitigating bias and enhancing group fairness without compromising downstream task performance. Moving forward, we acknowledge the importance of strengthening privacy protection when uploading group balance scores. One potential solution is the utilization of techniques such as homomorphic encryption, which enables computations on encrypted data and secure aggregation on the server side.
\bibliographystyle{alpha}
\bibliography{ref}

\clearpage
\newpage
\appendix
\section*{\centering Appendix} 
\section{Theoretical Analysis}\label{app:theo}
In this paper, we focus primarily on graph convolutional networks (GCNs). Given a non-linear function $g(\cdot)$, we consider a recursive relation in a GCN layer-wise propagation rule as $ \mH^l=g(\DAD\mH^{l-1}\W^l),$
where $\hat{\A}= \A + \mathbf{I}$, $\mathbf{I}$ is the identity matrix, $\hat{\D}$ is the diagonal node degree matrix $\hat{\A}$, $\W^l$ is the weight matrix in layer $l$ and $\mH^l \in \mathbb{R}^{N\times F^l}$ is the matrix of node embeddings at layer $l \in \{0,...,L\}$.
We note that the final layer $L$ of a GCN serves as the output layer, with $\Z \in \mathbb{R}^{N \times F^L}$ representing the output node representation, namely $\mH^L=\Z$, where $F^L$ denotes the output dimension, and $N$ denotes the size of the input data. GCNs are commonly used for representation learning, where node embeddings are generated by iteratively aggregating information from their neighbors to solve various learning tasks \cite{KipfW17}. Assuming that $\mathbf{h}^{l-1}_i$ is the representation of the $i$-th row of $\mH^{l-1}$ in layer $l-1$, the temporary aggregation representation through the $\DAD$ aggregation scheme can be expressed as ${\mathbf{h}^l_{agg}}_i = \frac{1}{d_i}\sum_{j \in \N(v_i)} \mathbf{h}^{l-1}_j$. Then, the node embedding $\mathbf{h}^{l}_i$ in layer $l$ is obtained by ${\mathbf{h}^l_{agg}}_i\,\w^l_i$, where $\w^l_i$ represents the $i$-th column of $\W^l$. 

\subsection{Calculation of $\mu_0$ and $\mu_1$}
\label{app:mu}
 The computation of means $\mu_0$ and $\mu_1$ ($\mu_j^{s=0}$ and $\mu_j^{s=1}$) can be formalized as follows:
\begin{equation*}
 \mu_j^{s=0} = \frac{1}{N_{s=0}} \sum_{j=1}^{N_{s=0}} \mathbf{X}_{j}^{s=0}\text{ and }\mu_j^{s=1} = \frac{1}{N_{s=1}} \sum_{j=1}^{N_{s=1}} \mathbf{X}_{j}^{s=1}, 
\end{equation*}
where $N_{s=0}$ ($N_{s=1}$) is the number of nodes with a sensitive attribute $s$.

\subsection{Proof of Lemma~\ref{lemma:rpb}}\label{proof:lemma}
\begin{mylemma}\ref{lemma:rpb}
    Given a graph $\G$ and a GNN model that satisfies the data and model assumptions, respectively, the Point-biserial correlation (denoted as $\pbcz$) between the graph representation  $\Z$ and the sensitive attribute $\mS$ is measured as the following:
 \begin{equation*}
        \pbcz = (\Nzero \mu_0 - \None \mu_1)(\intra -\inter)\cdot\frac{\sqrt{\Nzero\,\None}}{\SDZ\Ntotal^2},
 \end{equation*}
where $\intra$ and $\inter$ are defined by Definition \ref{def:edge_balance}, $\Nzero$ $(\None)$ is the number of nodes associated with $s=0$ $(s=1)$, $\SDZ$ is the standard deviation of $\Z$, $\Ntotal$ is the total number of nodes of $\G$, and $\mu_0$ ($\mu_1$) is the mean value of nodes that are associated with $s=0$ ($s=1$).
\end{mylemma}

\begin{proof}
According to the model assumption mentioned in Section~\ref{sec:rela_data_and_model_def},
the node representation $\Z$ will be expressed as a linear combination of the output from the last layer in GCNs without considering the non-linearity of activation function $g(\cdot)$: $\Z=g(\DAD\mH^{L-1}\W^{L-1})$. As a result, the node representation $\Z$ with different sensitive attribute $s$ can be modeled by two distinct Gaussian distributions, with mean vectors and covariance matrices denoted by $\mu'_0$, $\mu'_1$, $\Sigma'_0$, and $\Sigma'_1$, respectively:
\begin{align*}
    \bbP(\Z|s=0)\sim \mathcal{N}(\mu'_0, \Sigma'_0),\\
    \bbP(\Z|s=1)\sim \mathcal{N}(\mu'_1, \Sigma'_1).
\end{align*}

For convenience, we consider the feature matrix $\X$ only goes through one-layer GCNs and further simplify representation $\Z$ to $\tilde{\A}\X$, where $\tilde{\A}= \DAD$, when the number of GCNs' layer is more than 1, the derivation will be the same. Thus, in $m$-th column of the representation matrix:
\begin{equation*}
    \mu'_1=\bbE_{\Z_{:, m}}[\Z_{:, m}|s=1]= \bbE_{\Z_{:, m}}[\tilde{\A}\X_{:, m}|s=1].
\end{equation*}
Let $\inter = \bbP(s_i \ne s_j|\A_{ij}=1)$ and $\intra = \bbP(s_i=s_j|\A_{ij}=1)$, for any $z_i \in \Z_{:, m}$ and $x_i \in \X_{:, m}, m=1,.., d$, we have
\begin{equation*}
    \begin{aligned}
    z_i = & \sum_{j \in \N(x_i) \cap (s_i=s_j)}x_j\,\bbP(s_i = s_j|\A_{ij}=1) + \sum_{j \in \N(x_i) \cap (s_i\ne s_j)}x_j\, \bbP(s_i\ne s_j|\A_{ij}=1)\\
    =&\sum_{j \in \N(x_i) \cap (s_i=s_j)}x_j\,\intra + \sum_{j \in \N(x_i) \cap (s_i\ne s_j)}x_j\,\inter.
    \end{aligned}
\end{equation*}

Given a data size $\Nzero$ for any instance $x \in \X$ with a sensitive attribute $s=0$, and a data size $\None$ for any instance $x \in \X$ with a sensitive attribute $s=1$, we can express the expectation of a representation $z_i$ given $s=0$ as:
\begin{equation*}
    \begin{aligned}
    \bbE_\Z[z_i|_{s=0}]
    &=\frac{\Nzero\mu_0 \bbP(s_i = s_j|\A_{ij}=1) + \None\mu_1 \bbP(s_i\ne s_j|\A_{ij}=1)}{\Ntotal} \\
    &=\frac{\Nzero \mu_0 \intra + \None \mu_1 \inter}{\Ntotal} =\mu'_0,
\end{aligned}
\end{equation*}

Similarly, for $s=1$, the expectation of $z_i$ can be derived as:
\begin{equation*}
    \bbE_\Z[z_i|_{s=1}]=\frac{\None \mu_1 \intra + \Nzero \mu_0 \inter}{\Ntotal} =\mu'_1.
\end{equation*}

where $\Ntotal=\Nzero+\None$.
Define the standard deviation of node representation $\Z_{:, m}$ as ${\SDZ}_{:,i} = \sqrt{\frac{1}{N}\sum^n_{i=1}(z_i - \bar{\Z}_{:, m})^2}$, the point-biserial correlation coefficient of node representation and sensitive attribute can be expressed as:
\begin{equation*}
    \begin{aligned}
    \pbcz
    & = \frac{\mu'_0-\mu'_1 }{{\SDZ}_{:, m}} \sqrt{\frac{\Nzero \None}{\Ntotal^2}}\\
    & = \frac{\Nzero \mu_0 \intra + \None \mu_1 \inter - \None \mu_1 \intra - \Nzero \mu_0 \inter}{{\SDZ}_{:, m} \Ntotal}\cdot \sqrt{\frac{\Nzero \None}{\Ntotal^2}}\\
    &= \Big[ \Nzero \mu_0 (\intra - \inter) +\None \mu_1 (\inter - \intra) \Big] \cdot \theoitem\\
    & = \Big[\Nzero \mu_0 (\intra - \inter) - \None \mu_1 (\intra - \inter)\Big] \cdot \theoitem\\
    &= (\Nzero \mu_0 - \None \mu_1)(\intra - \inter)\cdot\theoitem.\label{eq:theo_result} 
\end{aligned}
\end{equation*}


Based on Equation~\eqref{eq:theo_result}, it can be observed that the distribution of edges reaches a completely balanced status when $\intra = \inter = 0.5$. Consequently, the point-biserial correlation coefficient, $\pbcz$, can attain a value of zero in such a scenario since $(\intra - \inter) = 0$. 
\end{proof}
However, when the distribution of sensitive attribute becomes balanced (in the case when $\sqrt{\Nzero\None}/N$ is maximized), the correlation $\pbcz$ is  minimized when $\inter=\intra$ or $\mu_0=\mu_1$.

\subsection{Proof of Theorem~\ref{theorm:rpb}}\label{proof:theorem}

\begin{mytheorem}\ref{theorm:rpb}
Let $\G$ be a graph that satisfies the given data assumption, and $\Z$ be the representation of $\G$ learned by a GNN model that satisfies the given model assumption. 
Then the Point-biserial correlation $\pbcz$ between the graph representation $\Z$ and the sensitive attribute $\mS$ is positively correlated/monotonically increasing with $|\intra - \inter|$. In particular, $\pbcz = 0$ when $\intra = \inter$. 
\end{mytheorem}

\begin{proof}
It is evident that the point-biserial correlation, denoted as $\pbcz$, between the graph representation $\Z$ and the sensitive attribute $\mS$, exhibits a strong linear correlation with the absolute difference $|\intra - \inter|$.
For convenience, we square $\pbcz$ to assess its magnitude. As $\pbcz^2$ approaches 0, the representation becomes more independent of the sensitive attribute. Conversely, as $\pbcz^2$ increases, the representation demonstrates a closer linear correlation with the sensitive attribute:
\begin{equation}
\label{eq:for proof theo}
    \pbcz^2 = (\Nzero \mu_0 - \None \mu_1)^2(\intra - \inter)^2\Big(\theoitem\Big)^2.
\end{equation}
Considering the range of values for both $\intra$ and $\inter$ within the interval [0, 1], while adhering to the constraint $\intra + \inter = 1$, we can observe that the function $(\intra - \inter)^2$ with the independent variable $|\intra - \inter|$ is monotonically increasing in the interval [0, 1]. Letting $d = |\intra - \inter|$, we have:
\begin{equation*}
     \frac{\partial d^2}{\partial d} = 2d,
\end{equation*}
and for all $d \in [0,1]$, $2d \geq 0$. Thus, the function $(\intra - \inter)^2$ with the independent variable $|\intra - \inter|$ is monotonically increasing in the interval [0, 1].

Assuming the other terms in Equation~\ref{eq:for proof theo} ($(\Nzero \mu_0 - \None \mu_1)^2$ and $\left(\theoitem\right)^2$) are fixed, we observe a monotonic relationship between $\pbcz^2$ and $(\intra - \inter)^2$. Consequently, we can deduce that the absolute value $|\pbcz|$ exhibits monotonicity with respect to the absolute difference $|\intra - \inter|$.
\end{proof}

\section{Algorithm}

\begin{breakablealgorithm}
\renewcommand{\algorithmicrequire}{\textbf{Input:}}
\renewcommand{\algorithmicensure}{\textbf{Output:}}
\caption{\ourmethod}
\label{alg:alg1}
\begin{algorithmic}[1]
\REQUIRE Graphs $\G_i(\V_i, \E_i)$, node features $\X_i$ on client $i$, where $i \in \{1, ..., K\}$ and $K$ is the total number of clients; hyperparameters $\alpha, \lambda_1, \lambda_2$; temperature parameter $\tau$; step-size $\eta$.
\ENSURE Global model $\omega^T$.
\STATE \textbf{Initialization:}
\STATE Initialize local model $\omega^i_0$ for each $i \in \{1, ..., K\}$ and global model $\omega_0$.
\FOR{$t = 1,2,\ldots,T$ \textbf{in parallel for all clients}}
\STATE \textbf{Local Update Phase:}
\STATE Randomly select $K'$ clients from the total client list.
\FOR{$i \in \{1,2,...,K'\}$}
\STATE Compute Jensen-Shannon divergence\\
$\jsold \leftarrow \mathrm{JS}(D_{\hat{\Y}_{\text{global}}}||D_{\hat{\Y}_{\text{local}}})$.
\STATE Update local model: $\omega^i_t \leftarrow \hat\omega^i_t - \eta \cdot \nabla \loss^i_{\text{total}_t}$,\\
where $\hat\omega^i_t \leftarrow (1-\jsold) \cdot \omega^i_{t-1} + \jsold \cdot \omega_{t-1}$ and\\ $\loss^i_{\text{total}_t} = \loss^i_{\text{fair}_t} + \alpha\,\loss^i_{\text{fair}_t}$.
\ENDFOR
\STATE \textbf{Global Update Phase:}
\STATE Construct merged weight vector $\gamma_t$ using the details in Section~\ref{sec:framework}.
\STATE Update global model $\omega_t \leftarrow \sum_{i=1}^{K'}\gamma^i_t \cdot \omega^i_t$.
\ENDFOR
\RETURN Global model $\omega^T$.
\end{algorithmic}
\end{breakablealgorithm}


\section{Dataset statistics}
\label{app:data}
\begin{table}[h]
\caption{Data Statistics.
}\label{tab:data} 
\centering
\resizebox{0.6\textwidth}{!}{
\begin{tabular}{@{}ccccc@{}}
\toprule
Data Set & \# Available Nodes& \# Edges & Sparsity & Nodes Overlapping\\    \midrule
Pokec-z &7659 & 48759 &   0.001662  & 25.06\%\\
Pokec-n &6185 &36827 & 0.001926 & 24.96\%\\
NBA &310 &14540 & 0.303580  & 100\%\\
\bottomrule
\end{tabular}
}
\end{table}
Table~\ref{tab:data} presents an overview of the dataset statistics. In this context, ``Available nodes'' refers to the count of nodes that are part of the largest connected subgraph within a given dataset. Concurrently, ``Edges'' delineates the total count of edges that are integral to this maximal connected subgraph.
The sparsity of the graph is computed adhering to the conventional definition:
\begin{equation*}
\text{Sparsity} = \frac{2 |\E|}{|\V|(|\V|-1)},
\end{equation*}
where $|\V|$ symbolizes the total count of nodes and 
$|\E|$ represents the total count of edges present in the graph.
Furthermore, the ``Nodes Overlapping'' metric is instrumental in quantifying the proportion of nodes that are shared amongst two or more distinct clients.

\section{Details of Experimental Setup}
\label{app:exp_detail}

For all baseline models, node representations are derived using a two-layer Graph Convolutional Network (GCN). Subsequently, an $\ell_2$-regularized logistic regression-based linear classifier is employed for node classification tasks, as utilized by Dai \textit{et al.}~\cite{DaiW21}. In our experimental design, we adopt a random data partitioning strategy: 50\% of the nodes form the training set, while the remaining nodes are equally split into validation and test sets. Each experimental run is executed over five unique data splits, with the outcomes expressed as an average, accompanied by standard deviations. 
 The experimental environment utilized for our evaluations is as follows:
\begin{itemize}
    \item Operating system: Linux version 5.15.0-71-generic
    \item CPU information: Intel Xeon E5-2698v4 @ 2.20GHz
    \item GPU information: NVIDIA Titan V
\end{itemize}
With regards to hyperparameters, we adjust the temperature parameter $\tau$ within the interval $[0.01, 4]$. The parameters $\lambda$ and $\alpha$ are explored within the range $[2, 4]$. All model parameters are initialized randomly, and the optimization of the model parameters is carried out using the Adam optimizer.
\section{Additional Experiments}
\label{app:etra_exp}
\subsection{Performance Evaluation on NBA Dataset}
As shown in Table~\ref{tab:results_NBA}, the results on the NBA dataset demonstrate that $\text{F}^2$GNN outperforms FairAug$_{\text{+FL}}$ and FairFed in terms of accuracy, AUC, $\Delta_{\text{SP}}$, and $\Delta_{\text{EO}}$ on the global test set. Specifically, $\text{F}^2$GNN achieves the lowest $\Delta_{\text{SP}}$ and $\Delta_{\text{EO}}$ scores while maintaining competing accuracy and AUC scores, indicating its strong performance in achieving a balance between predictive accuracy and fairness. These findings highlight the effectiveness of $\text{F}^2$GNN in node classification tasks for smaller datasets.

We must highlight a specific characteristic of the NBA dataset used in our study, as shown in Table~\ref{tab:data}. Given that it comprises only 150 available nodes, we face a challenge in the federated graph setting to avoid assigning an empty graph to a client. Consequently, this limitation led to a 100\% node overlapping. In this context, both our proposed method and the FairFed baseline yield identical results on the global test set and the local test sets. However, a distinct behavior is observed when employing FairAug$_{\text{+FL}}$. This approach incorporates data augmentation techniques to alter the edge structure within each client. While this strategy enhances the diversity of the data, it simultaneously amplifies the disparity between each client's data distribution and the global data distribution. As a result, FairAug$_{\text{+FL}}$ manifests a pronounced difference in performance between the local and global modes.

\begin{table}[t]
\caption{Node classification and fairness evaluation on the test sets of the \textbf{NBA} datasets.}
\label{tab:results_NBA}
\centering
\resizebox{0.55\textwidth}{!}{
\setlength{\tabcolsep}{2.6pt}
\begin{tabular}{@{}lcccc@{}}
\toprule
&  \multicolumn{4}{c}{{\bf Global}} \\ \cmidrule{2-5}
 Method  & Accuracy(\%) $\uparrow$ & AUC(\%)  $\uparrow$ & $\Delta_{\text{SP}}$(\%) $\downarrow$ & $\Delta_{\text{EO}}$(\%) $\downarrow$ \\ \midrule
FairAug$_{\text{+FL}}$ & $67.14 \pm 0.17$        & $67.26 \pm 0.18$       & $8.84 \pm 4.22$              & $3.86 \pm 1.79$\\ 
FairFed & $68.20 \pm 13.48$       & $74.79\pm 19.99$      & $9.52 \pm 7.62$             & $6.16 \pm 7.80$\\ 
$\text{F}^2$GNN & $\mathbf{73.08 \pm 1.29}$       & $\mathbf{79.59 \pm 1.44}$   & $\mathbf{3.54 \pm 1.29}$              & $\mathbf{2.02 \pm 0.65}$\\ 
\midrule
&  \multicolumn{4}{c}{{\bf Local}} \\ \cmidrule{2-5}
 Method  & Accuracy(\%) $\uparrow$ & AUC(\%)  $\uparrow$ & $\Delta_{\text{SP}}$(\%) $\downarrow$ & $\Delta_{\text{EO}}$(\%) $\downarrow$ \\ \midrule
FairAug$_{\text{+FL}}$ & $67.14 \pm 0.17$        & $67.26 \pm 0.18$       & $8.84 \pm 4.22$              & $3.86 \pm 1.79$\\ 
FairFed & $68.20 \pm 13.48$       & $74.79\pm 19.99$      & $9.52 \pm 7.62$             & $6.16 \pm 7.80$\\ 
$\text{F}^2$GNN & $\mathbf{73.08 \pm 1.29}$       & $\mathbf{79.59 \pm 1.44}$   & $\mathbf{3.54 \pm 1.29}$              & $\mathbf{2.02 \pm 0.65}$\\ 
\bottomrule
\end{tabular}
}
\end{table}
\begin{table}[t]
\caption{Node classification and fairness performance comparison with non-federated baselines.}
\label{tab:results_non_fed}
\centering
\setlength{\tabcolsep}{2.6pt}  
\resizebox{0.6\textwidth}{!}{
\begin{tabular}{lcccccccc}
\toprule
& \multicolumn{5}{c}{Pokec-z} \\
\cmidrule{3-6}
Method     & FL & Accuracy(\%) $\uparrow$ & AUC(\%) $\uparrow$   & $\Delta_{\text{SP}}$ (\%) $\downarrow$ & $\Delta_{\text{EO}}$(\%) $\downarrow$ \\ \midrule
FairGCN    & \xmark & $65.48\pm1.42 $              & $71.94\pm1.19$  & $4.83\pm2.17$                  & $4.46\pm2.90$   \\
FairVGNN   & \xmark & $64.04\pm2.23$              & $72.54\pm0.4$  & $5.83\pm2.3$                  & $6.8\pm1.27$    \\
EDITS      & \xmark & $64.33\pm1.1 $              & $69.39\pm0.98$ & $5.91\pm1.93$                 & $6.78\pm2.86$   \\
FairAug    & \xmark & $67.26\pm0.52$              & $67.21\pm0.52$ & $5.11\pm2.13$                 & $5.2\pm2.14$    \\
\midrule
\ourmethod  & \cmark & $\mathbf{68.17\pm0.19}$              & $\mathbf{73.47\pm0.52}$ & $\mathbf{1.66\pm1.02}$                 & $\mathbf{1.49\pm0.52}$ \\
\midrule
& \multicolumn{5}{c}{Pokec-n} \\
\cmidrule{3-6}
Method     & FL & Accuracy(\%) $\uparrow$ & AUC(\%) $\uparrow$   & $\Delta_{\text{SP}}$ (\%) $\downarrow$ & $\Delta_{\text{EO}}$(\%) $\downarrow$ \\ \midrule
FairGCN    & \xmark & $64.99\pm1.20$              & $68.52\pm 1.48$ & $3.09\pm2.69$                 & $7.34\pm3.55$   \\
FairVGNN   & \xmark & $64.42\pm1.2$               & $68.55\pm0.25$ & $4.64\pm1.07$                 & $3.79\pm1.38$   \\
EDITS      & \xmark & $62.43\pm1.54$              & $66.68\pm1.56$ & $5.84\pm4.26$                 & $6.62\pm5.43$   \\
FairAug    & \xmark & $\mathbf{67.22\pm0.26}$              & $64.36\pm0.41$ & $3.33\pm0.77$                 & $5.47\pm1.18$   \\
\midrule
\ourmethod  & \cmark & $67.00\pm0.27$               & $\mathbf{69.62\pm1.34}$  & $\mathbf{0.85\pm0.31}$                 & $\mathbf{1.00\pm1.03}$  \\
\midrule
&  \multicolumn{5}{c}{NBA} \\ \cmidrule{3-6}
 Method  & FL & Accuracy(\%) $\uparrow$ & AUC(\%)  $\uparrow$ & $\Delta_{\text{SP}}$(\%) $\downarrow$ & $\Delta_{\text{EO}}$(\%) $\downarrow$
 \\ \midrule
 FairGCN  & \xmark              & $70.00 \pm 5.78$           & $74.83 \pm 3.68$    & $18.45 \pm 16.04$               & $25.44 \pm 20.35$\\
FairVGNN & \xmark            &$67.17 \pm 1.30$          &$\mathbf{82.43 \pm 0.14}$      &$15.71 \pm 3.02$               &$13.75 \pm 1.53$\\
EDITS & \xmark               & $70.77 \pm 4.29$            & $75.22 \pm 4.03$   & $14.15 \pm 8.61$              & $8.22 \pm 4.40$\\
\midrule
$\text{F}^2$GNN & \cmark  & $\mathbf{73.08 \pm 1.29}$       & $79.59 \pm 1.44$   & $\mathbf{3.54 \pm 1.29}$              & $\mathbf{2.02 \pm 0.65}$\\ 
\bottomrule
\end{tabular}
}
\end{table}



\subsection{Performance Comparison with Non-Federated Baselines}

We compare our method with four state-of-the-art fairness-aware graph neural network (GNN) methodologies: FairGNN \cite{DaiW21}, FairVGNN \cite{WangZDCLD22}, EDITS \cite{DDongLJL22}, and FairAug \cite{KoseS22}. \footnote{We utilize the implementations of FairGNN available at {\url{https://github.com/EnyanDai/FairGNN}},\\
FairVGNN at {\url{https://github.com/yuwvandy/fairvgnn}},\\ 
EDITS at {\url{https://github.com/yushundong/EDITS}},\\
and FairAug at {\url{https://openreview.net/forum?id=4pijrj4H_B}}.} These approaches were originally designed for non-federated settings. The evaluation results on Pokec-z, Pokec-n, and NBA datasets are presented in Table \ref{tab:results_non_fed}. When we run FairAug on the NBA dataset, the algorithm fails to converge, and the accuracy remains around 50\%. Therefore, we exclude the results of FairAug from Table~\ref{tab:results_non_fed}.

In the Pokec-z dataset, \ourmethod\ outperforms all baselines in terms of accuracy, AUC, and fairness metrics. For the Pokec-n dataset, \ourmethod\ achieves competitive accuracy and AUC compared to the baselines, with superior performance in fairness metrics. In the NBA dataset, \ourmethod\ achieves the highest accuracy and AUC, with significantly lower fairness metrics than other baselines.

Note that in the non-federated settings, we have the original complete graph information of the training data. However, in federated settings, there is some loss of graph data due to the split of data across different clients. 


\section{Limitations and Broader Impact}
Our work assumes access to sensitive attributes during training and adopts a simplified theoretical analysis when we have binary values for sensitive attributes. We think this research has broader impacts on fairness-aware machine learning, with potential benefits across various applications. Yet, the implementation of our method demands careful consideration of context-specific fairness definitions, ethical implications, and the risk of reverse discrimination. Understanding application-specific requirements and integrating domain knowledge is essential for responsible use.

\end{document}